\documentclass{article}

\usepackage{microtype}
\usepackage{graphicx}
\usepackage{booktabs} 

\usepackage{hyperref}

\usepackage[accepted]{icml2020}
\usepackage[utf8]{inputenc} 
\usepackage[T1]{fontenc}    
\usepackage{hyperref}       
\usepackage{url}            
\usepackage{booktabs}       
\usepackage{amsfonts}       
\usepackage{nicefrac}       
\usepackage{microtype}      
\usepackage{amsfonts}       
\usepackage{nicefrac}       
\usepackage{microtype}      
\usepackage{times}
\usepackage{adjustbox}
\usepackage{tabularx}
\usepackage{amsmath}
\usepackage{multirow}
\usepackage{comment}
\usepackage{amsthm}
\usepackage{blindtext}
\usepackage{float}
\usepackage{enumitem}
\usepackage{multirow}
\usepackage{graphicx}
\usepackage{bm}
\usepackage{amsfonts}
\usepackage{amsthm}
\usepackage{subcaption}
\theoremstyle{definition}
\newtheorem{assumption}{Assumption}
\newtheorem{definition}{Definition}
\newtheorem{theorem}{Theorem}

\newtheorem{lemma}{Lemma}
\newtheorem{remark}{Remark}
\usepackage{color}
\icmltitlerunning{DP-SCO with Heavy-tailed Data}

\begin{document}

\twocolumn[
\icmltitle{On Differentially Private Stochastic Convex Optimization \\
with Heavy-tailed Data}



\icmlsetsymbol{equal}{*}

\begin{icmlauthorlist}
\icmlauthor{Di Wang}{equal,to,ed}
\icmlauthor{Hanshen Xiao}{equal,goo}
\icmlauthor{Srini Devadas}{goo}
\icmlauthor{Jinhui Xu}{to}
\end{icmlauthorlist}

\icmlaffiliation{to}{Department of Computer Science and Engineering, State University of New York at Buffalo, Buffalo, NY}
\icmlaffiliation{goo}{CSAIL, MIT, Cambridge, MA}
\icmlaffiliation{ed}{King Abdullah University of Science and Technology, Thuwal, Saudi Arabia}

\icmlcorrespondingauthor{Di Wang}{dwang45@buffalo.edu}

\icmlkeywords{Machine Learning, ICML}

\vskip 0.3in
]



\printAffiliationsAndNotice{\icmlEqualContribution} 

\begin{abstract}
In this paper, we consider the problem of designing 
Differentially Private (DP) algorithms for Stochastic Convex Optimization (SCO) on heavy-tailed data. The irregularity of such data violates some key 
assumptions used in almost all existing DP-SCO and DP-ERM methods, resulting in failure to provide the DP guarantees. To better understand this type of challenges, we provide in this paper a comprehensive study of DP-SCO under various settings. First, we consider the case where the loss function is strongly convex and smooth. For this case, we propose a method based on the sample-and-aggregate framework, which has an excess population risk of 
$\tilde{O}(\frac{d^3}{n\epsilon^4})$ (after omitting other factors), where $n$ is the sample size and $d$ is the dimensionality of the data. Then, we show that with some additional assumptions on the loss functions, it is possible to reduce the 
\textit{expected} excess population risk to $\tilde{O}(\frac{ d^2}{ n\epsilon^2 })$. To lift these additional conditions, we also  
 provide a gradient smoothing and trimming based scheme to achieve  excess population risks of   
$\tilde{O}(\frac{ d^2}{n\epsilon^2})$ and $\tilde{O}(\frac{d^\frac{2}{3}}{(n\epsilon^2)^\frac{1}{3}})$ for strongly convex and general convex loss functions, respectively, \textit{with high probability}. Experiments suggest that our algorithms can effectively deal with the challenges caused by data irregularity.
\end{abstract}
\section{Introduction}
Stochastic Convex Optimization (SCO)  \citep{vapnik2013nature}  and its empirical form, Empirical Risk Minimization (ERM), are the most fundamental problems in supervised learning and statistics. They find numerous applications in many areas such as medicine, finance, genomics  and social science. One often encountered challenge in such models is how to handle  sensitive data, such as those in biomedical datasets. As a commonly-accepted approach for preserving privacy, differential privacy \citep{dwork2006calibrating}  provides provable protection against identification and is resilient to
arbitrary auxiliary information that might be available to
attackers. Methods to guarantee differential privacy have
been widely studied, and recently adopted in industry \citep{apple,ding2017collecting}. 

Differentially Private Stochastic Convex Optimization and Empirical Risk Minimization ({\em i.e.,} DP-SCO and DP-ERM) have been
extensively studied in the past decade, 
starting from
\citep{chaudhuri2009privacy,chaudhuri2011differentially}.  Later on, 
a long list of works have attacked the problems
from different perspectives: \citep{bassily2014private, wang2017differentially,wang2019differentially,wu2017bolt,bassily2020} studied the problems in the low dimensional case and the central model, \citep{kasiviswanathan2016efficient,kifer2012private,talwar2015nearly} considered the problems in the high dimensional sparse case and the central model, \citep{smith2017interaction,wang2018empirical,wang2019noninteractive,duchi2013local} focused on the problems in the local model.

It is worth noting that all previous results need to assume that either the loss function is $O(1)$-Lipschitz  or each data sample has bounded  $\ell_2$ or $\ell_{\infty}$ norm. This is particularly true for those output perturbation based \citep{chaudhuri2011differentially} and objective or gradient perturbation based 
\citep{bassily2014private} DP methods.
However, such assumptions may not always hold when dealing with real-world datasets, especially those from biomedicine and finance, implying that existing algorithms may fail. 
The main reason is that in such applications, the datasets are often unbounded or even heavy-tailed 
\citep{woolson2011statistical,biswas2007statistical,ibragimov2015heavy}. 
As pointed out by Mandelbrot and Fama in their influential finance papers 
\citep{mandelbrot1997variation,fama1963mandelbrot}, asset prices in the early 1960s exhibit some power-law behavior.  
The heavy-tailed data could lead to unbounded gradient and thus violate the Lipschitz condition. 
For example, consider the linear squared loss $\ell(w,x, y)=(w^Tx-y)^2$. When $x$ is heavy-tailed, the gradient of $\ell(w, x,y)$ becomes unbounded. 


With the above understanding, our questions now are:  
{\bf  What is the behavior of DP-SCO on heavy-tailed data and is there any effective method for the problem?}

To answer these questions, we will conduct, in this paper,
a comprehensive study of the DP-SCO problem. Our contributions can be summarized as follows. 
\begin{enumerate}
    \item We first consider the case where the loss function is strongly convex and smooth. For this case, we propose an $(\epsilon, \delta)$-DP method based on the sample-and-aggregate framework by \citep{nissim2007smooth} and show that under some assumptions, with high probability, the excess population risk of the output is $\tilde{O}(\frac{d^3}{n\epsilon^4}L_\mathcal{D} (w^*))$, where $n$ is the sample size, $d$ is the dimensionality and $L_\mathcal{D} (w^*)$ is the minimal value of the population risk.
    
    \item Then, we study the case with the additional assumptions:  
    each coordinate of the gradient of the loss function is sub-exponential and Lipschitz. For this case, we introduce an $(\epsilon, \delta)$-DP algorithm based on the gradient descent method and a recent algorithm on private 1-dimensional mean estimation  \citep{bun2019average} ({\em i.e.,} Algorithm \ref{alg:3}). We show that the expected excess population risk for this case can be improved to $\tilde{O}(\frac{ d^2 \log \frac{1}{\delta}}{ n\epsilon^2 })$. 
    
    \item We also consider the general case, where the loss function does not need the above additional assumptions and can be general convex, instead of strongly convex. For this case, we present a gradient descent method based on the strategy of trimming the unbounded gradient (Algorithm \ref{alg:4}). 
    We show that
    if each coordinate of the gradient of the loss function has bounded second-order moment, then with high  probability, the output of our algorithm achieves excess population risks of $\tilde{O}(\frac{ d^2  \log \frac{1}{\delta}}{n\epsilon^2})$ and $\tilde{O}(\frac{\log \frac{1}{\delta }d^\frac{2}{3}}{(n\epsilon^2)^\frac{1}{3}})$ for strongly convex and general convex loss functions, respectively. It is notable that compared with Algorithm \ref{alg:4}, Algorithm \ref{alg:3} uses stronger assumptions and yields weaker results. 
    
    \item Finally, we test our proposed aglorithms on both synthetic and real-world datasets. Experimental results are consistent with our theoretical claims and reveal the effectiveness of our algorithms in handling heavy-tailed datasets. 
\end{enumerate}
Due to the space limit, some definitions, all the proofs are relegated to the appendix in the Supplementary Material, which also includes the codes of experiments. 
\section{Related Work}
As mentioned earlier, there is a long list of works on DP-SCO or DP-ERM. However, none of them considers the case with heavy-tailed data. 
Recently, a number of works have studied the SCO and ERM problems with heavy-tailed data 
\citep{brownlees2015empirical,minsker2015geometric,hsu2016loss,lecue2018robust}. However, all of them focus on the non-private version of the problem. 
It is not clear whether they can be adapted to private versions. 
To our best knowledge, the work presented in this paper is the  first one on general DP-SCO with heavy-tailed data. 

The works that are most related to ours are perhaps those dealing with 
unbounded sensitivity. \citep{dwork2009differential} proposed a general framework called propose-test-release and applied it to mean estimation. They obtained asymptotic results which are incomparable with ours. Also, it is not clear whether such a framework can be applied to our problem. In our second result, we adopt the private mean estimation procedure in
\citep{bun2019average}. However, their results are in expectation form, which is not preferred in robust estimation \citep{brownlees2015empirical}. For this reason, we propose a new algorithm which 
yields theoretically guaranteed bounds with high probability.
\citep{karwa2017finite} considered the confidence interval estimation problem for Gaussian distributions which was later extended to general distributions \citep{feldman2018calibrating}. However, it was unknown how to extend them to the DP-SCO problem. 
\citep{abadi2016deep} proposed a DP-SGD method based on truncating the gradient, which could deal with the infinity sensitivity issue. However, there is no theoretical guarantees on the excess population risk. 
\section{Preliminaries}

\begin{definition}[Differential Privacy \citep{dwork2006calibrating}]\label{def:3.1}
	Given a data universe $\mathcal{X}$, we say that two datasets $D,D'\subseteq \mathcal{X}$ are neighbors if they differ by only one entry, which is denoted as $D \sim D'$. A randomized algorithm $\mathcal{A}$ is $(\epsilon,\delta)$-differentially private (DP) if for all neighboring datasets $D,D'$ and for all events $S$ in the output space of $\mathcal{A}$, the following holds
	\[\mathbb{P}(\mathcal{A}(D)\in S)\leq e^{\epsilon} \mathbb{P}(\mathcal{A}(D')\in S)+\delta.\]
\end{definition}

	\begin{definition}[DP-SCO \citep{bassily2014private}]\label{definition:1}
		Given a dataset $D=\{x_1,\cdots,x_n\}$ from a data universe $\mathcal{X}$ where $x_i$ are i.i.d. samples from some unknown distribution $\mathcal{D}$, a convex loss function $\ell(\cdot, \cdot)$, and a convex constraint set  $\mathcal{W} \subseteq \mathbb{R}^d$, Differentially Private Stochastic Convex Optimization (DP-SCO) is to find $w^{\text{priv}}$ so as to minimize the population risk, {\em i.e.,} $L_\mathcal{D} (w)=\mathbb{E}_{x\sim \mathcal{D}}[\ell(w, x)]$
		with the guarantee of being differentially private.
		 The utility of the algorithm is measured by the \textit{(expected) excess population risk}, that is  $\mathbb{E}_{\mathcal{A}}[L_\mathcal{D} (w^{\text{priv}})]-\min_{w\in \mathbb{\mathcal{W}}}L_\mathcal{D} (w),$
where the expectation of $\mathcal{A}$ is taken over all the randomness of the algorithm. Besides 
the population risk, we can also measure the \textit{empirical risk} of dataset $D$: $\hat{L}(w, D)=\frac{1}{n}\sum_{i=1}^n \ell(w, x_i).$
	\end{definition}
	
	\begin{definition}
	    A random variable $X$ with mean $\mu$ is called $\tau$-sub-exponential if $\mathbb{E}[\exp(\lambda (X-\mu))]\leq \exp(\frac{1}{2}\tau^2\lambda^2), \forall |\lambda|\leq \frac{1}{\tau}$.
	\end{definition}
	\begin{definition}
	    A function $f$ is $L$-Lipschitz if  for all $w, w'\in\mathcal{W}$, $|f(w)-f(w')|\leq L\|w-w'\|_2$.
	\end{definition}
	\begin{definition}
	    A function $f$ is $\alpha$-strongly convex on $\mathcal{W}$ if for all $w, w'\in \mathcal{W}$, $f(w')\geq f(w)+\langle \nabla f(w), w'-w \rangle+\frac{\alpha}{2}\|w'-w\|_2^2$.
	\end{definition}
	\begin{definition}
	    	    A function $f$ is $\beta$-smooth on $\mathcal{W}$ if for all $w, w'\in \mathcal{W}$, $f(w')\leq f(w)+\langle \nabla f(w), w'-w \rangle+\frac{\beta}{2}\|w'-w\|_2^2$.
	\end{definition}
	\begin{assumption}\label{ass:1}
	For the loss function and the population risk, we assume the following.
	\begin{enumerate}
	    \item The loss function $\ell(w, x)$ is non-negative,  differentiable and convex for all $w\in \mathcal{W}$ and $x \in \mathcal{X}$. 
	    \item The population risk $L_{\mathcal{D}}(w)$ is $\beta$-smooth. 
	    \item The convex constraint set $\mathcal{W}$ is bounded with diameter $\Delta=\max_{w, w'\in \mathcal{W}}\|w-w'\|_2< \infty$.
	    \item The optimal solution $w^*=\arg\min_{w\in \mathcal{W}} L_\mathcal{D}(w)$ satisfies $\nabla L_\mathcal{D}(w^*)= 0$. 
	\end{enumerate}
	\end{assumption}
	\begin{assumption} \label{ass:2}
	 There exists a number $n_\alpha$ such that when the sample size $|D|\geq n_\alpha$,  the empirical risk $\hat{L}(\cdot, D)$ is $\alpha$-strongly convex with probability at least $\frac{5}{6}$ over the choice of i.i.d. samples in $D$. 
	\end{assumption}
We note that Assumptions \ref{ass:1} and \ref{ass:2} are commonly used in the studies on the problem of Stochastic Strongly Convex Optimization with heavy-tailed data, such as \citep{hsu2016loss,holland2019a}.  Also the probability of $\frac{5}{6}$ in Assumption \ref{ass:2} is only for convenience.
\begin{assumption}\label{ass:3}
We assume the following for the loss functions. 
\begin{enumerate}
    \item For any $w\in \mathcal{W}$ and each coordinate $j\in [d]$, we assume that the random variable $\nabla_j \ell(w, x)$ is $\tau$-sub-exponential and $\beta_j$-Lipschitz (that is $\ell_j(w, x)$ is $\beta_j$-smooth), where $\nabla_j$ represents the $j$-th coordinate of the gradient. 
    \item  There are known constants $a, b = O(1)$ such that $a \leq \mathbb{E}[\nabla_j \ell(w, x)]\leq b$ for all $w\in \mathcal{W}$. 
\end{enumerate}
\end{assumption}
\begin{assumption}\label{ass:4}
For any $w\in \mathcal{W}$ and each coordinate $j\in [d]$, we have $\mathbb{E}[(\nabla_j \ell(w, x))^2]\leq v=O(1)$, where $v$ is some known constant. 
\end{assumption}
We can see that, compared with Assumption \ref{ass:3}, Assumption \ref{ass:4} needs fewer assumptions on the loss functions, because we only need to assume the gradient of the loss function has bounded second-order moment. We also note that Assumption \ref{ass:4} is more suitable to the  problem of 
Stochastic Convex Optimization with heavy-tailed data and has been used in some previous works such as \citep{holland2017efficient,brownlees2015empirical}.
\section{Sample-aggregation based method}
In this section we first summarize the sample-aggregate framework introduced in \citep{nissim2007smooth}. 

Most of the existing privacy-preserving frameworks are based on the notion of \textit{global sensitivity}, which is defined as the maximum output perturbation $\|f(D)-f(D')\|_{\xi}$, where the maximum is over all neighboring datasets $D, D'$ and $\xi=1,2$. However, in some problems such as clustering \citep{nissim2007smooth,wang2015differentially} the sensitivity could be very high and thus ruin the utility of the algorithm. 

To circumvent this issue, \citep{nissim2007smooth} introduced the sample-aggregate framework based on 
a smooth version of \textit{local sensitivity}. Unlike the global sensitivity, local sensitivity measures the maximum perturbation $\|f(D)-f(D')\|_\xi$ over all databases $D'$ neighboring the input database $D$. The proposed sample-aggregate framework (Algorithm \ref{alg:1}) enjoys local sensitivity and comes with the following guarantee: 
\begin{theorem}[Theorem 4.2 in \citep{nissim2007smooth}]\label{thm:1}
Let $f: \mathcal{D}\mapsto \mathbb{R}^d$ be a function where $\mathcal{D}$ is the collection of all databases and $d$ is the dimensionality of the output space. Let $d_{\mathcal{M}}(\cdot, \cdot)$ be a semi-metric on the output space of $f$. Set $\epsilon> \frac{2d}{\sqrt{m}}$ and $m=\omega(\log^2 n)$. The sample-aggregate algorithm $\mathcal{A}$ in Algorithm \ref{alg:1} is an efficient $(\epsilon, \delta)$-DP algorithm.\footnote{Here the efficiency means that the time complexity is polynomial in all terms.} Furthermore, if $f$ and $m$ are chosen such that the $\ell_1$ norm of the output of $f$ is bounded by $\Lambda$ and
\begin{equation}\label{eq:1}
    \text{Pr}_{D_S\subseteq D}[d_{\mathcal{M}}(f(D_S), c)\leq r]\geq \frac{3}{4}
\end{equation}
for some $c\in \mathbb{R}^d$ and $r>0$, then the standard deviation of Gaussian noise added is upper bounded by $O(\frac{r}{\epsilon}+\frac{\Lambda}{\epsilon}e^{-\Omega(\frac{\epsilon\sqrt{m}}{d})}).$ In addition, when $m=\omega(\frac{d^2\log^2(r/\Lambda)}{\epsilon^2})$, with high probability each coordinate of $\mathcal{A}(D)-\bar{c}$ is upper bounded by $O(\frac{r}{\epsilon})$, where $\bar{c}$ depending on $\mathcal{A}(D)$ satisfies $d_{\mathcal{M}}(c, \bar{c})=O(r)$.  
\end{theorem}
\begin{algorithm}[h]
	\caption{Sample-aggregate Framework \citep{nissim2007smooth}} \label{alg:1}
	$\mathbf{Input}$: $D=\{x_i\}_{i=1}^n\subset \mathbb{R}^d$, number of subsets $m$, privacy parameters $\epsilon, \delta$; $f, d_{\mathcal{M}}$.  
	\begin{algorithmic}[1]
   \STATE {\bf Initialize:} $s=\sqrt{m}, \gamma=\frac{\epsilon}{5\sqrt{2\log(2/\delta)}}$ and $\beta= \frac{\epsilon}{4(d+\log(2/\delta))}$. 
   \STATE {\bf Subsampling:} Select $m$ random subsets of size $\frac{n}{m}$ of $D$ independently and uniformly at random without replacement. Repeat this step until no single data point appears in more than $\sqrt{m}$ of the sets. Mark the subsampled subsets $D_{S_1}, D_{S_2}, \cdots, D_{S_m}$.
   \STATE Compute $\mathcal{S}=\{s_i\}_{i=1}^m$,  where $s_i=f(D_{S_i})$.
   \STATE Compute $g(\mathcal{S})=s_{i^*}$, where $i^*=\arg\min_{i=1}^m r_i(t_0)$ with $t_0=\frac{m+s}{2}+1$. Here $r_i(t_0)$ denotes the distance $d_{\mathcal{M}}(\cdot, \cdot)$ between $s_i$ and the $t_0$-th nearest neighbor to $s_i$ in $\mathcal{S}$. 
   \STATE {\bf Noise Calibration:}  Compute $S(\mathcal{S})=2\max_{k}(\rho(t_0+(k+1)s)\cdot e^{-\beta k}),$ where $\rho(t)$ is the mean of the top $\lceil \frac{s}{\beta} \rceil$ values in $\{r_1(t), \cdots, r_m(t)\}$. \\
   \STATE Return $\mathcal{A}(D)=g(\mathcal{S})+\frac{S(\mathcal{S})}{\gamma}u$, where $u$ is a standard Gaussian random vector. 
	\end{algorithmic}
\end{algorithm}
We have the following Lemma \ref{lemma:3}, which shows that the minimum of the empirical risk satisfies (\ref{eq:1}). 
\begin{lemma}\label{lemma:3}
Let $w_D=f(D)=\arg\min_{w\in \mathcal{W}}\hat{L}(w, D)$ where $|D|=n$. Then, under Assumptions \ref{ass:1} and \ref{ass:2},  if $n\geq n_\alpha$, the following holds 
\begin{equation}\label{eq:2}
   \text{Pr}[\|w_D-w^*\|_2\leq \eta ] \geq \frac{3}{4},
\end{equation}
where $\eta=O(\sqrt{\frac{\mathbb{E}\|\nabla \ell(w^*, x)\|_2^2 }{n\alpha^2}})$. 
\end{lemma}

Combining Lemma \ref{lemma:3} and Theorem \ref{thm:1}, we get the following upper bound for DP-SCO with heavy-tailed data and strongly convex loss functions. 
\begin{theorem}\label{thm:2}
Under Assumptions \ref{ass:1} and \ref{ass:2},  for any $\epsilon, \delta>0$, if $n\geq \tilde{\Omega}(\frac{n_\alpha d^2}{\epsilon^2})$,  $m \geq \tilde{\omega}(\frac{d^2}{\epsilon^2})$, $f(D)=\arg\min_{w\in \mathcal{W}}\hat{L}(w, D)$ and $d_\mathcal{M} (x, y) = \|x-y\|_2$, 
then Algorithm \ref{alg:1} is $(\epsilon, \delta)$-DP. Moreover, with high probability the output of $\mathcal{A}(D)$ ensures that 
\begin{equation}\label{eq:3}
    L_\mathcal{D}(\mathcal{A}(D))-L_\mathcal{D}(w^*)\leq \tilde{O}((\frac{\beta}{\alpha})^2 \frac{d^3}{n\epsilon^4}L_\mathcal{D} (w^*)),
\end{equation}
where the Big-$\tilde{O}, \Omega$ and small-$\omega$ notations omit the logarithmic terms. 
\end{theorem}

\begin{remark}
For DP-SCO with Lipschitz and strongly-convex loss function and bounded data,  \citep{bassily2014private,wang2017differentially,bassily2020} showed that the upper bound of the excess population risk is $O(\frac{\sqrt{d}}{n\epsilon})$, and the lower bound is $\Omega(\frac{d}{n^2\epsilon^2})$ \footnote{\citep{bassily2014private} only shows the lower bound of the excess empirical risk. 
We can obtain the lower bound of the excess population risk by using the reduction from private ERM to private SCO  \citep{bassily2020}.}. This suggests that the bound in Theorem \ref{thm:2} has some additional factors related to $d$ and $\frac{1}{\epsilon}$.
We note that 
the upper bound in Theorem \ref{thm:2} has a multiplicative term of $L_\mathcal{D} (w^*)$. This means that when $L_\mathcal{D} (w^*)$ is small, our bound is better. For example, when $L_\mathcal{D} (w^*)=0$,  our algorithm can  recover $w^*$ exactly and results in an excess risk of $0$. Notice that there is no previous work on DP-ERM or DP-SCO that has a multiplicative error with respect to  $L_\mathcal{D} (w^*)$. 
\end{remark}
\section{Gradient descent based methods}
There are several issues in the sample-aggregation based method presented in last section. 
Firstly, function $f(D)$ in Theorem \ref{thm:2} needs to  solve the optimization problem exactly, which could be quite inefficient in practice. Second, previous empirical evidence suggests that sample-aggregation based methods often suffer from poor utility in practice \citep{su2016differentially,wang2015differentially}. Thirdly, Theorem \ref{thm:2} needs to assume  strong convexity for the empirical risk and it is unclear whether it can be extended  to the general convex case.
Finally, from Eq.(\ref{eq:3}) we can see that when $L_\mathcal{D}(w^*)=\Theta(1)$,  the excess population risk is quite large as compared to the ones in \citep{bassily2014private}. Thus, an immediate question is whether we can further lower the upper bound. To answer this question and resolve the above issues, we propose in this section two  DP algorithms based on the Gradient Descent method under different assumptions. 

Recently, \cite{bun2019average} studied the problem of estimating the mean of a $1$-dimensional heavy-tailed distribution and proposed  algorithms based on the idea of truncating the empirical mean and the local sensitivity. Motivated by this DP algorithm that has the capability of handling heavy-tailed data, we plan to develop a new method by borrowing some ideas from the work 
 \citep{bun2019average} and robust gradient descent. 
 Our method  is inspired by their theorem that follows and uses the  Arsinh-Normal mechanism (see Algorithm \ref{alg:2} and Prop. 5 in \citep{bun2019average}).
 
\begin{theorem}[Theorem 7 in \citep{bun2019average}]\label{thm:3}
Let $0<\epsilon, \delta\leq 1$ be two constants and $n$ be some integer $\geq O(\log(\frac{n(b-a)/\sigma}{\epsilon})$. Then, there exists a $\frac{1}{2}\epsilon^2$-zero concentrated Differentially Private (zCDP) (see Appendix for the definition of zCDP) algorithm (Algorithm \ref{alg:2}) $M:\mathbb{R}^n \mapsto \mathbb{R}$ such that the following holds: Let $\mathcal{D}$ be a distribution with mean $\mu \in [a, b]$, where $a,b$ are given constants and unknown variance $\sigma^2$. Then, 
\begin{equation*}
    \mathbb{E}_{X\sim \mathcal{D}^n, Z}[(M(X)-\mu)^2]\leq O(\frac{\sigma^2\log n}{n\epsilon^2}). 
\end{equation*}
\end{theorem}

The key idea of our algorithm is that, in each iteration, after getting $w^{t-1}$, we use the mechanism in Theorem \ref{thm:3} on each coordinate of $\nabla \ell(w, x_i)$. See Algorithm  \ref{alg:3} for details. 
\begin{algorithm}[!h]
	\caption{Mechanism $\mathcal{M}$ in \citep{bun2019average}} \label{alg:2}
	$\mathbf{Input}$: $D=\{x_i\}_{i=1}^n\subset \mathbb{R}, \epsilon, a, b.$ 
	\begin{algorithmic}[1]
    \STATE Let $t=\frac{\epsilon^2}{16}$ and $s=\frac{\epsilon}{4}$. Sort $\{x_i\}_{i=1}^n$ in the ascending order as $x_{(1)}\leq x_{(2)}\leq \cdots \leq x_{(n)}$. Calculate the upper bound of the smooth sensitivity for the trimming and truncating step:
    \begin{equation*}
        S^{t}_{[\text{trim}_m(\cdot)]_{[a, b]}}(D)= \max\{ \frac{x_{(n)}-x_{(1)}}{n-2m}, e^{-mt}(b-a)\}, 
    \end{equation*}
    where $m=O(1)\leq \frac{n}{2}$ is a constant. 
    \STATE Do the average trimming and truncating step:
    \begin{equation*}
        [\text{Trim}_{m}(D)]_{[a,b]}=[ \frac{x_{(m+1)}+\cdots+x_{(n-m)}}{n-2m}]_{[a,b]}, 
    \end{equation*}
    where $[x]_{[a,b]}= x$ if $a\leq x \leq b$, equals to $a$ if $x< a$ and otherwise equals to $b$. 
    \STATE Output $[\text{Trim}_{m}(D)]_{[a,b]}+ \frac{1}{s} S^{t}_{[\text{trim}_m(\cdot)]_{[a, b]}}(D)\cdot Z$, where $Z=\text{sinh}(Y)=\frac{e^Y-e^{-Y}}{2}$ and $Y$ is the Standard Gaussian. 
	\end{algorithmic}
\end{algorithm}
\begin{algorithm}[!h]
	\caption{Heavy-tailed DP-SCO with known mean} \label{alg:3}
	$\mathbf{Input}$: $D=\{x_i\}_{i=1}^n\subset \mathbb{R}^d$, privacy parameters $\epsilon, \delta$; loss function $\ell(\cdot, \cdot)$,  initial parameter $w^0$, $a, b$ which satisfy Assumption \ref{ass:3}, and the number of iterations $T$ (to be specified later).
	\begin{algorithmic}[1]
    \STATE Let $\tilde{\epsilon}= \sqrt{2\log \frac{1}{\delta}+2\epsilon}-\sqrt{2\log \frac{1}{\delta}}$. 
    \FOR {$t=1, 2, \cdots, T$}
    \STATE For each $j\in [d]$,  calculate 
    
    $D_{t-1, j}(w^{t-1})= \{\nabla_j \ell(w^{t-1}, x_i)\}_{i=1}^n$.
    \STATE Run Algorithm \ref{alg:2} for each $D_{t-1, j}$ and denote the output 
    
    $\tilde{\nabla}_{t-1, j}(w^{t-1})=(\mathcal{M}(D_{t-1,j}(w^{t-1})), \frac{\tilde{\epsilon}}{\sqrt{d T}}, a, b)$.  Denote 
    $$\nabla \tilde{L}(w^{t-1}, D)= (\tilde{\nabla}_{t-1, 1}(w^{t-1}) \cdots, \tilde{\nabla}_{t-1, d}(w^{t-1})).$$
    \STATE Updating $w^t=\mathcal{P}_\mathcal{W} (w^{t-1}- \eta_{t-1}\nabla \tilde{L}(w^{t-1}, D))$, where $\eta_{t-1}$ is some step size and $\mathcal{P}_\mathcal{W}$ is the projection operator.
    \ENDFOR
	\end{algorithmic}
\end{algorithm}
By the composition theorem and the relationship between $zCDP$ and $(\epsilon, \delta)$-DP \citep{bun2016concentrated}, we have the DP guarantee. 
\begin{theorem}\label{thm:4}
For any $0<\epsilon, \delta\leq 1$, Algorithm \ref{alg:3} is $(\epsilon, \delta)$-differentially private. 
\end{theorem}
To show the \textit{expected} excess population risk of Algorithm \ref{alg:3}, we cannot use the upper bound in Theorem \ref{thm:3} directly for the following reasons. First, since the upper bound is for the expectation w.r.t. $X$ and $Z$ while the \textit{expected} excess population risk depends only on the randomness of the algorithm instead of the data. Thus, we need to obtain an upper bound for $\mathbb{E}_{ Z}[(M(X)-\mu)^2]$ (with high probability w.r.t. $X$). Secondly, to get an upper bound, it is sufficient to analyze the term $\|\nabla \tilde{L}(w^{t-1}, D)-\nabla L_\mathcal{D}(w^{t-1})\|_2$ in each iteration. However, since the parameter $w^{t-1}$ at any step depends on the random draw of the dataset $\{x_i\}_{i=1}^n$, upper bounds on the estimation error need to be uniform in $w\in \mathcal{W}$ in order to capture all contingencies. To resolve these two issues, we use the same technique as in \citep{chen2017distributed,vershynin2010introduction} (under Assumption \ref{ass:3}) to obtain the following lemma. 
\begin{lemma}\label{lemma:4}
Under Assumption \ref{ass:3}, 
with probability at least $1-\frac{2dn}{(1+n\hat{\beta}\Delta)^d}$ the following holds  for all $w\in \mathcal{W}$,
\begin{equation}\label{eq:4}
    \mathbb{E}_{Z}\| \nabla \tilde{L}(w, D)- \nabla L_\mathcal{D}(w)\|_2\leq O( \frac{\tau d\sqrt{T\log n}}{\sqrt{n}\tilde{\epsilon}}), 
\end{equation}
where $\hat{\beta}=\sqrt{\beta_1^2+\cdots+\beta_d^2}$, the expectation is w.r.t. the random variables $\{Z_i\}_{i=1}^d$ and the Big-$O$ notation omits other factors.
\end{lemma}
Next, we show the expected excess population risk for strongly convex loss  functions.

\begin{theorem}[Strongly-convex case]\label{thm:5}
Under Assumptions \ref{ass:1} and \ref{ass:3}, 
if the population risk is $\alpha$-strongly convex and $T$ and $\eta$ are set to be $T=O(\frac{\beta}{\alpha}\log n)$ and $\eta=\frac{1}{\beta}$, respectively, in  Algorithm \ref{alg:3}, then 
with probability at least $1-\Omega(\frac{\beta}{\alpha}\frac{2dn\log n}{(1+n\hat{\beta}\Delta)^d})$ the output satisfies the following for all $D\sim \mathcal{D}^n$, 
\begin{equation*}\label{eq:5}
   \mathbb{E}[ L_\mathcal{D}(w^T)]-L_\mathcal{D} (w^*)\leq O(\frac{\Delta^2\beta^2\tau^2 d^2 \log^2 n \log \frac{1}{\delta}}{\alpha^3 n\epsilon^2 }).
\end{equation*}
\end{theorem}

Compared with the bound in Theorem \ref{thm:2}, we can see that the bound in Theorem \ref{thm:5} improves a factor of $\tilde{O}(\frac{d}{\epsilon^2})$ (if we omit other terms).  However, there are more assumptions on the distribution and the loss functions. Specifically, in Assumption \ref{ass:3} we need to assume the sub-exponential property, {\em i.e.,} the moment of $\nabla_j\ell(w, x)$ exists for every order. Also, we need to assume that $\nabla_j\ell(w, x)$ is Lipschitz and the range of its mean is known. These assumptions are quite strong, compared to those used in the literature of learning with heavy-tailed data, such as \citep{holland2017efficient,brownlees2015empirical,hsu2016loss,minsker2015geometric}. 

To improve the above result, we consider the following. First, we would like to  
relax those assumptions in the theorem. Second, in the problem of ERM with heavy-tailed data, it is expected to have an excess population risk bound that is in the form of \textit{with high probability} instead of its \textit{expectation}  \citep{brownlees2015empirical}. However, it is unclear whether Algorithm \ref{alg:3} can achieve a high probability bound. This is due to the fact that the noise added in each iteration is  a combination of log-normal distributions, which is non-sub-exponential and thus is hard to get tail bounds. Third, Algorithm \ref{alg:3} depends on the local sensitivity and thus cannot be extended to the distributed settings or local differential privacy model.  Finally, the practical performance of Algorithm \ref{alg:3} has poor utility and is unstable due to the noise added in each iteration (see Section 6 for details), which means that Algorithm \ref{alg:3} is still impractical. To resolve all these issues and still keeping
(approximately) the same upper bound, we propose a new algorithm that is simply based on the Gaussian mechanism. 

In the following we will study the problem under Assumptions 1 and \ref{ass:4}. Note that compared with Assumption \ref{ass:3}, we only need to assume that the second-order moment of $\nabla_j\ell(w, x)$ exists for all $w\in\mathcal{W}$ and $j\in [d]$ and its upper bound is known.

Our method is motivated by the robust mean estimator given in \citep{holland2019a}. To be self-contained, we first review their estimator. Now, we consider  $1$-dimensional random variable $x$ and assume that $x_1, x_2, \cdots, x_n$ are i.i.d. sampled from $x$. The estimator consists of the following steps: 
\paragraph{Scaling and Truncation} For each sample $x_i$, we first re-scale it by dividing $s$ (which will be specified later). Then, we apply the re-scaled one to some soft truncation function $\phi$. Finally, we put the truncated mean back to the original scale. That is, 
\begin{equation}\label{eq:6}
    \frac{s}{n}\sum_{i=1}^n \phi(\frac{x_i}{s})\approx \mathbb{E}X.
\end{equation}
Here, we use the function given in \citep{catoni2017dimension},
\begin{equation}\label{eq:7}
    \phi(x)= \begin{cases} x-\frac{x^3}{6}, & -\sqrt{2}\leq x\leq \sqrt{2} \\
    \frac{2\sqrt{2}}{3}, & x>\sqrt{2} \\
    -\frac{2\sqrt{2}}{3}, & x<-\sqrt{2}.   
    \end{cases}
\end{equation}
Note that a key property for $\phi$ is that $\phi$ is bounded, that is,  $|\phi(x)|\leq \frac{2\sqrt{2}}{3}$.
\paragraph{Noise Multiplication} Let $\eta_1, \eta_2, \cdots, \eta_n$ be random noise generated from a common distribution $\eta\sim \chi$ with $\mathbb{E}\eta =0$. We multiply each data $x_i$ by a factor of $1+\eta_i$, and then perform the scaling and truncation step on the term $x_i(1+\eta_i)$.  
That is, 
\begin{equation}\label{eq:8}
    \tilde{x}(\eta) =\frac{s}{n}\sum_{i=1}^n \phi(\frac{x_i+\eta_i x_i}{s}).
\end{equation}
\paragraph{Noise Smoothing} In this final step, we smooth the multiplicative noise by taking the expectation w.r.t. the distributions. That is, 
\begin{equation}\label{eq:9}
    \hat{x}=\mathbb{E}  \tilde{x}(\eta) = \frac{s}{n}\sum_{i=1}^n \int \phi(\frac{x_i+\eta_i x_i}{s})d \chi(\eta_i).
\end{equation}
Computing the explicit form of each integral in (\ref{eq:9}) depends on the function $\phi(\cdot)$ and the distribution $\chi$. Fortunately, \cite{catoni2017dimension} showed that when $\phi$ is in (\ref{eq:7}) and $\chi\sim \mathcal{N}(0, \frac{1}{\beta})$ (where $\beta$ will be specified later), we have for any $a, b$
\begin{equation}\label{eq:10}
    \mathbb{E}_{\eta}\phi(a+b\sqrt{\beta}\eta)=a(1-\frac{b^2}{2})-\frac{a^3}{6}+C(a,b),
\end{equation}
where $C(a, b)$ is a correction form which is easy to implement and its explicit form will be given in the Appendix. 

\begin{algorithm*}[!ht]
	\caption{Heavy-tailed DP-SCO with known variance} \label{alg:4}
	$\mathbf{Input}$: $D=\{x_i\}_{i=1}^n\subset \mathbb{R}^d$, privacy parameters $\epsilon, \delta$, loss function $\ell(\cdot, \cdot)$,  initial parameter $w^0$, $v$ which satisfies Assumption \ref{ass:4}, the number of iterations $T$ (to be specified later), and failure probability $\delta'$. 
	\begin{algorithmic}[1]
    \STATE Let $\tilde{\epsilon}= (\sqrt{\log \frac{1}{\delta}+\epsilon}-\sqrt{\log \frac{1}{\delta}})^2$, $s=\sqrt{\frac{nv}{2\log \frac{1}{\delta'}}}$, $\beta=\log \frac{1}{\delta'}$. 
    \FOR {$t=1, 2, \cdots, T$}
    \STATE For each $j\in [d]$, calculate the robust gradient by (\ref{eq:8})-(\ref{eq:10}), that is 
    \begin{multline}
        g_j^{t-1}(w^{t-1})= \frac{1}{n}\sum_{i=1}^n \left(\nabla_j\ell(w^{t-1}, x_i)\big(1-\frac{\nabla^2_j\ell(w^{t-1}, x_i)}{2s^2\beta}\big)- \frac{\nabla^3_j\ell(w^{t-1}, x_i)}{6s^2}\right)\\+\frac{s}{n}\sum_{i=1}^nC\left(\frac{\nabla_j\ell(w^{t-1}, x_i)}{s}, \frac{|\nabla_j\ell(w^{t-1}, x_i)|}{s\sqrt{\beta}}\right)+ Z_{j}^{t-1},
    \end{multline}
    where $Z_{j}^{t-1}\sim \mathcal{N}(0, \sigma^2)$ with $\sigma^2= \frac{8vdT}{9\log \frac{1}{\delta'}n\tilde{\epsilon}}$.
    \STATE Let vector $g^{t-1}(w^{t-1})\in \mathbb{R}^d$ to denote  $g^{t-1}(w^{t-1})=(g_1^{t-1}(w^{t-1}), g_2^{t-1}(w^{t-1}), \cdots, g_d^{t-1}(w^{t-1}))$. 
    \STATE Update
        $w^{t}=\mathcal{P}_{\mathcal{W}}(w^{t-1}-\eta_{t-1}g^{t-1}). $
    \ENDFOR
	\end{algorithmic}
\end{algorithm*}

 \cite{holland2019a} showed the following estimation error  for  the mean estimator $\hat{x}$ after these three steps. 
\begin{lemma}[Lemma 5 in \citep{holland2019a}] \label{lemma:5}
Let  $x_1, x_2, \cdots, x_n$ be i.i.d. samples from  distribution $x\sim \mu$. Assume that there is some known upper bound on the second-order moment, {\em i.e.,} $\mathbb{E}_\mu x^2\leq v$. For a given failure probability $\delta'$, if  set  $\beta= 2\log \frac{1}{\delta'}$ and $s=\sqrt{\frac{nv}{2\log\frac{1}{\delta'}}}$, then with probability at least $1-\delta'$ the following holds 
\begin{equation}
    |\hat{x}-\mathbb{E}x|\leq O(\sqrt{\frac{v\log \frac{1}{\delta'}}{n}}). 
\end{equation}
\end{lemma}
To obtain an $(\epsilon,\delta)$-DP estimator, the key observation is that the bounded function $\phi$ in (\ref{eq:7})  also makes the integral form of  (\ref{eq:10}) bounded by $\frac{2\sqrt{2}}{3}$. Thus, we know that the $\ell_2$-norm sensitivity is $\frac{s}{n}\frac{4\sqrt{2}}{3}$. Hence, the query 
\begin{equation}\label{eq:14}
    \mathcal{A}(D)=\hat{x}+ Z, Z\sim \mathcal{N}(0, \sigma^2), \sigma^2=O(\frac{s^2\log \frac{1}{\delta}}{\epsilon^2n^2})
\end{equation}
will be $(\epsilon, \delta)$-DP, which leads to  
the following theorem.

\begin{theorem}\label{theorem:6}
Under the assumptions in Lemma \ref{lemma:5}, with probability at least $1-\delta'$ the following holds
\begin{equation}
    |\mathcal{A}(D)-\mathbb{E}(x)|\leq O(\sqrt{\frac{v\log \frac{1}{\delta}\log\frac{1}{\delta'}}{n\epsilon^2}}).
\end{equation}
\end{theorem}

Comparing with Theorem \ref{thm:3}, we can see that the upper bound in Theorem \ref{theorem:6} is in the form of `with high probability' (after transferring zCDP to $(\epsilon, \delta)$-DP \citep{bun2016concentrated}). Moreover, we improve by a factor of $O(\log n)$ in the error bound. 

Inspired by Theorem \ref{theorem:6} and Algorithm \ref{alg:3}, we propose a new method (Algorithm \ref{alg:4}), which uses our private mean estimator (\ref{eq:14}) on each coordinate of the gradient in each iteration. The following theorem shows the error bound when the loss function is strongly convex. 

\begin{theorem}\label{thm:7}
For any $0<\epsilon, \delta<1$, Algorithm \ref{alg:4} is $(\epsilon, \delta)$-DP. Under Assumptions \ref{ass:1} and \ref{ass:4},  if the population risk is $\alpha$-strongly convex and $\eta_t$ and $T$  
in Algorithm \ref{alg:4} are set to be  $\eta_t=\frac{1}{\beta}$ and $T=O(\frac{\beta}{\alpha}\log n)$, respectively, then for any $\delta'>0$, with probability at least $1-2\delta' T$ the output $w^{T}$ satisfies 
\begin{equation*}
       L_\mathcal{D}(w^T)-L_\mathcal{D} (w^*)\leq O(\frac{v\Delta^2\beta^4 d^2 \log^2 n \log \frac{1}{\delta}\log \frac{1}{\delta'}}{\alpha^3 n\epsilon^2}).
\end{equation*}
\end{theorem}
Comparing with Theorem \ref{thm:7} and \ref{thm:5}, we can see that if we omit other terms, the bounds are asymptotically the same and Theorem  \ref{thm:7} needs fewer assumptions. 

With the high probability guarantee on the error in Theorem \ref{theorem:6}, we can actually get an upper bound for general convex loss functions. For this general convex case, we need the following mild technical assumption on the constraint set $\mathcal{W}$. 

\begin{assumption}\label{ass:5}
The constraint set $\mathcal{W}$ contains the following $\ell_2$-ball centered at $w^*$: $\{w: \|w-w^*\|_2\leq 2\|w^0-w^*\|_2\}$. 
\end{assumption}

\begin{theorem}[Convex case]\label{thm:8} Under  Assumptions \ref{ass:1}, \ref{ass:4} and \ref{ass:5}, if we take $\eta=\frac{1}{\beta}$ and $T=\tilde{O}\left(\frac{\|w^0-w^*\|_2\sqrt{n}\sqrt{\tilde{\epsilon}}}{d}\right)^\frac{2}{3}$ in Algorithm \ref{alg:4}, then for any given failure probability  $\delta'$, with probability at least $1-T\delta'$  the following holds 
\begin{equation}
     L_\mathcal{D}(w^T)-L_\mathcal{D} (w^*)\leq \tilde{O}(\frac{\log^\frac{1}{3} \frac{1}{\delta}\sqrt{\log \frac{1}{\delta' }}d^\frac{2}{3}}{(n\epsilon^2)^\frac{1}{3}}) 
\end{equation}
when $n\geq \tilde{\Omega}(\frac{d^2}{\epsilon^2})$, 
where the Big-$\tilde{O}$ notation omits other logarithmic factors and the term of $v, \beta$. 
\end{theorem}
\section{Experiments}

\paragraph{Baseline Methods} As mentioned earlier, sample-aggregation based methods often have poor practical performance. Thus, we will not conduct experiments on Algorithm \ref{alg:1}. Moreover, as this is the first paper studying DP-SCO with heavy-tailed data and almost all previous methods on DP-SCO that have theoretical guarantees fail to provide DP guarantees,  we do not compare  our methods with them, and instead focus on comparing 
the performance of Algorithm \ref{alg:3} and Algorithm \ref{alg:4}. 
To show the effectiveness of our methods, we use the non-private heavy-tailed SCO method in \citep{holland2019a}, denoted by (stochastic) RGD in the following, as our baseline method.
\paragraph{Experimental Settings} 
For synthetic data,  we consider the linear and  binary logistic models. Specifically, we generate the synthetic datasets in the following way. Each dataset has a size of  
$1\times 10^5$ and 
each data point $(x_i, y_i)$   is generated by the model of $y_i = \langle \omega^*, x_i \rangle + e_i$ and $y_i = \text{sign}[\frac{1}{1+e^{\langle \omega^*, x_i \rangle + e_i}}-\frac{1}{2}]$, respectively, where
 $x_i \in \mathbb{R}^{10}$ and $y_i \in \mathbb{R}$.  
 In the first model, the zero mean noise $e_i$ is generated as follows. 
We first generate a noise $\Delta_i$ from  the $(\mu, \sigma)$ log-normal distribution, {\em i.e.,} $\mathbb{P}(\Delta_i = x ) = \frac{1}{x\sigma\sqrt{2\pi}} e^{-\frac{(\ln x -\mu)^2}{2\sigma^2}}$, and then  let $e_i = \Delta_i - \mathbb{E}[\Delta_i]$. For the second model, we first generate a noise $\Delta_i$ from the $(\mu, \sigma)$ log-logistic distribution, {\em i.e.,} $\mathbb{P}(\Delta_i = x ) = \frac{e^z}{\sigma x(1+e^z)^2}$, where $x>0$ and $z = \frac{\log(x)-\mu }{\sigma}$. Then, we let $e_i = \Delta_i - \mathbb{E}[\Delta_i]$. Accordingly, we  implement Algorithm \ref{alg:3} and Algorithm \ref{alg:4}, together 
with RGD, on the ridge and logistic regressions. 


For real-world data, we use the Adult dataset from the UCI Repository \citep {Dua:2019}. 
We aim to predict whether the annual income of an individual is above 50,000. 
We select 30,000 samples, 28,000 amongst which are used as the training set and the rest are used for test. 

For the privacy parameters, we will choose $\epsilon=\{0.1, 0.5, 1\}$ and $\delta=O(\frac{1}{n})$. See Appendix for the selections of other parameters. For Algorithm \ref{alg:3}, the strength of prior knowledge is modeled by $\kappa=b-a$.
\paragraph{Experimental Results}
Figure \ref{fig:1} and \ref{fig:2} show the results of ridge and logistic regressions on synthetic and real datasets w.r.t iteration, respectively. Since there is no ground truth in the real dataset, we use the empirical risk on test data as the measurement. To test scalability of Algorithm \ref{alg:4} dealing with large-scaling data, experiments on stochastic versions of Algorithm \ref{alg:4} and RGD with minibatch size 1000 are also conducted. We can see that the performance of Algorithm \ref{alg:3} bears a larger variation compared to Algorithm 4, since we have to apply a heavy-tailed noise to fit the smooth sensitivity.  Moreover, the performance of Algorithm \ref{alg:3} is sensitive to the parameter $\kappa$. Thus, these results show that Algorithm \ref{alg:3} has poor performance and the results of Algorithm \ref{alg:4} are comparable to the non-private ones. In Figure \ref{fig:3} and \ref{fig:4} we test the estimation error  w.r.t different dimensionality $d$ and sample size $n$, respectively. From these results we can see that when $n$ increases or $d$ decreases, the estimation error will decrease. Also, with fixed $n$ and $d$, we can see that the estimation error will decrease as $\epsilon$ becomes larger. Thus, all these results confirm our previous theoretical analysis.

\begin{figure*}[!htbp]
    \centering
    \begin{subfigure}[b]{.24\textwidth}
    \includegraphics[width=\textwidth,height=0.15\textheight]{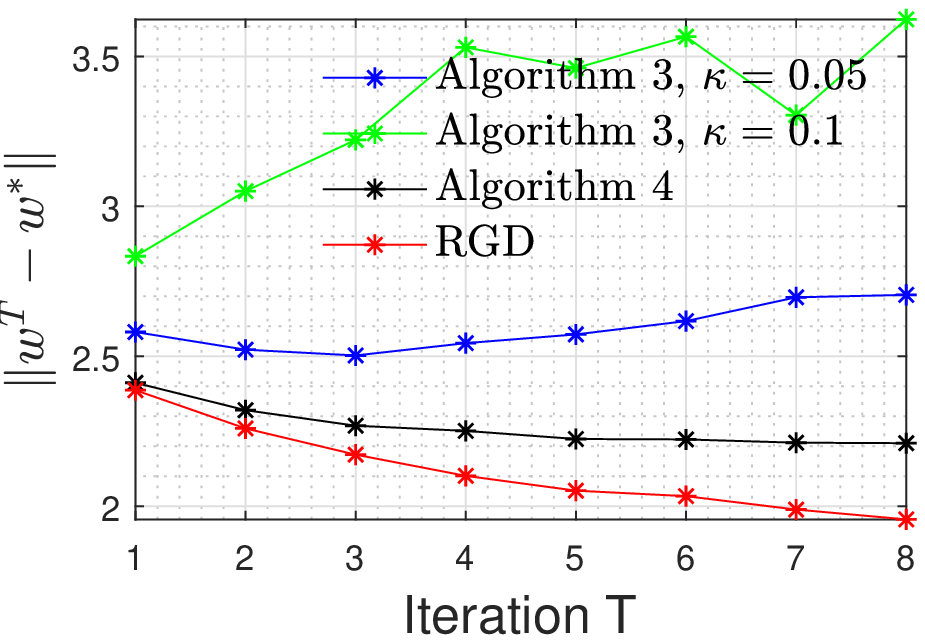}
    \caption{$\epsilon=1$ \label{fig1:a}}
    \end{subfigure}
    ~
    \begin{subfigure}[b]{.24\textwidth}
    \includegraphics[width=\textwidth,height=0.15\textheight]{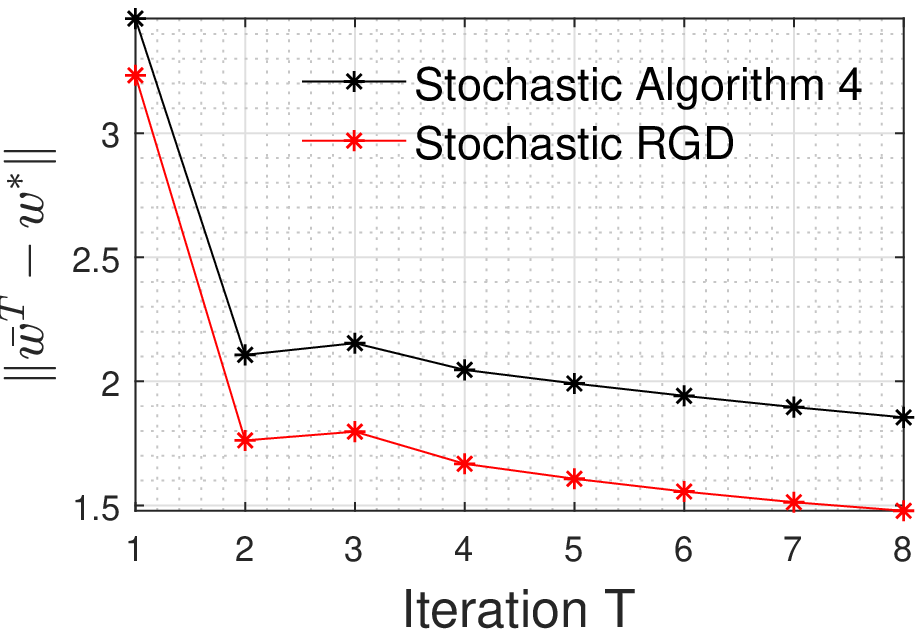}
    \caption{$\epsilon=0.5$ \label{fig1:b}}
    \end{subfigure}    
     \begin{subfigure}[b]{.24\textwidth}
    \includegraphics[width=\textwidth,height=0.15\textheight]{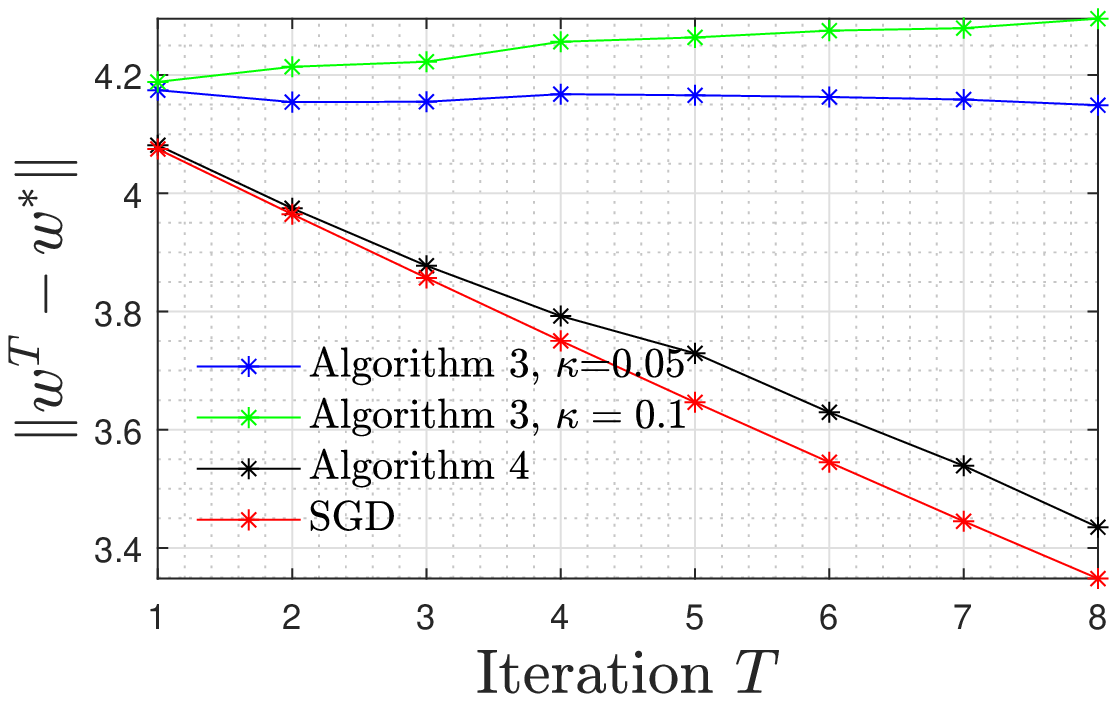}
    \caption{$\epsilon=1$ \label{fig1:c}}
    \end{subfigure}
    ~
    \begin{subfigure}[b]{.24\textwidth}
    \includegraphics[width=\textwidth,height=0.15\textheight]{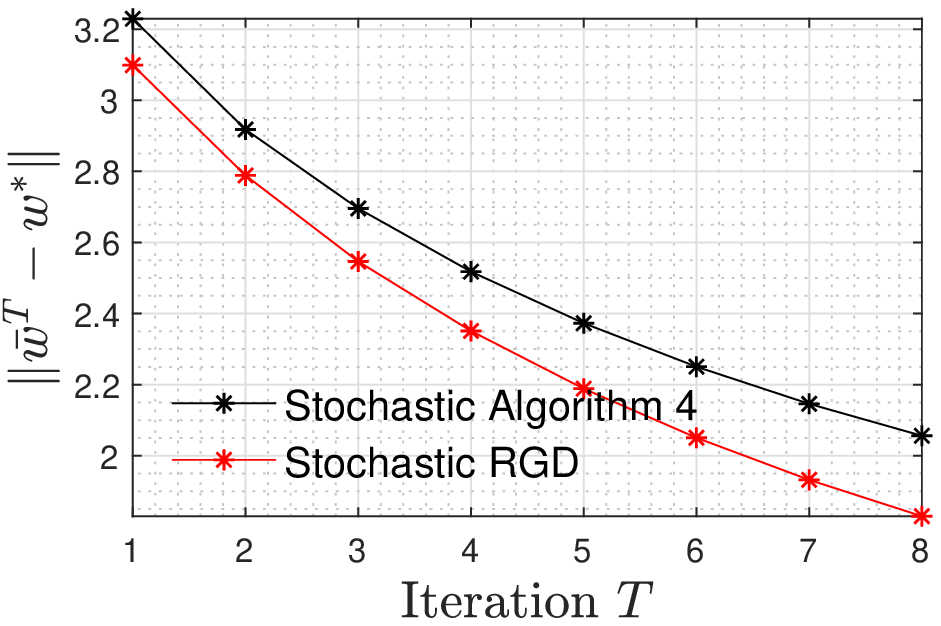}
    \caption{$\epsilon=0.5$ \label{fig1:d}}
    \end{subfigure}  
    \caption{Experiments on synthetic datasets. Figures \ref{fig1:a} and \ref{fig1:b} are for ridge regressions over synthetic data with Lognormal noises. Figures \ref{fig1:c} and \ref{fig1:d} are for logistic regressions over synthetic data with Loglogistic noises. \label{fig:1} }
\end{figure*}
\begin{figure*}[!htbp]
    \centering
    \begin{subfigure}[b]{.24\textwidth}
    \includegraphics[width=\textwidth,height=0.15\textheight]{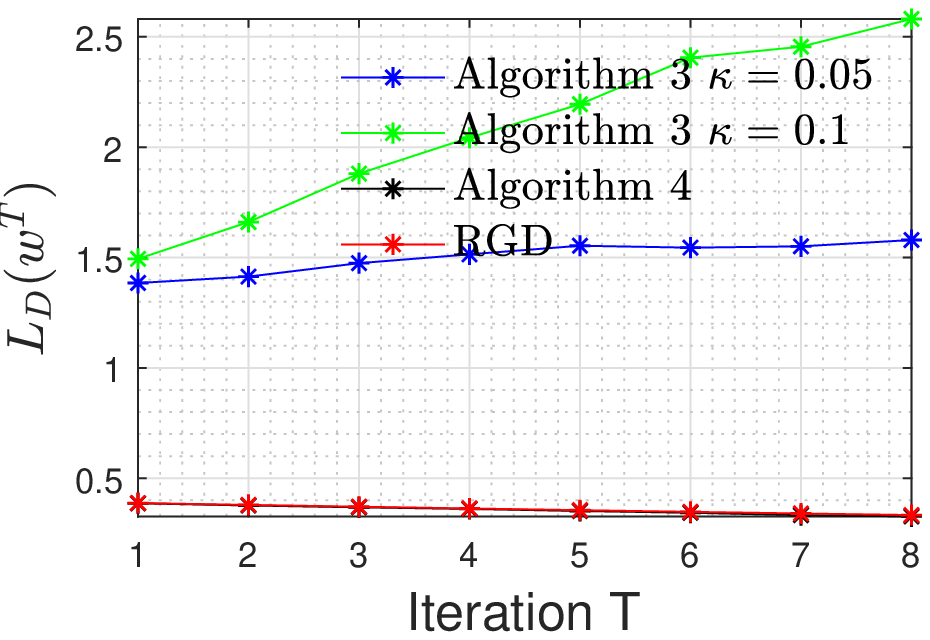}
    \caption{$\epsilon=1$ \label{fig2:a}}
    \end{subfigure}
    ~
    \begin{subfigure}[b]{.24\textwidth}
    \includegraphics[width=\textwidth,height=0.15\textheight]{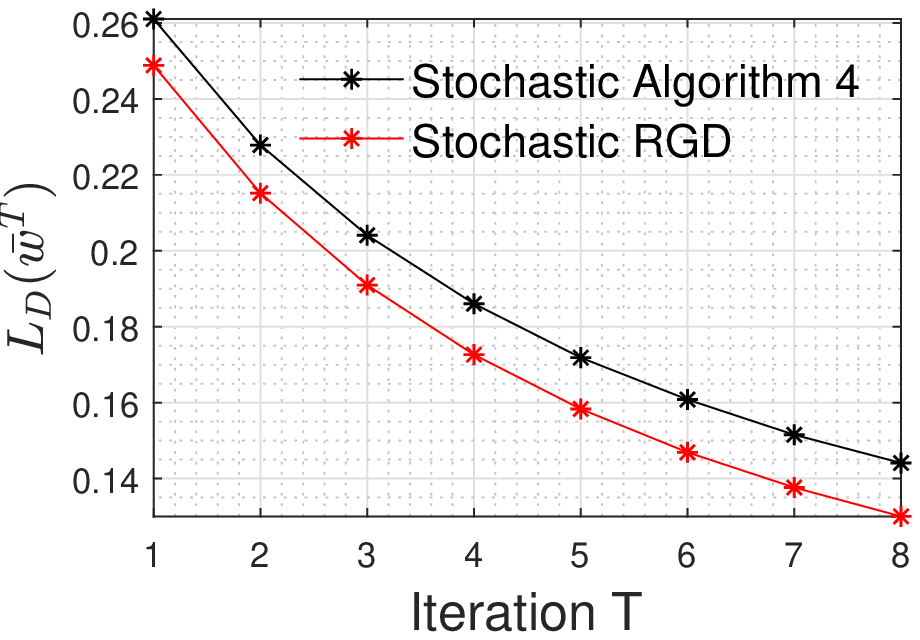}
    \caption{$\epsilon=0.5$ \label{fig2:b}}
    \end{subfigure}    
       \begin{subfigure}[b]{.24\textwidth}
    \includegraphics[width=\textwidth,height=0.15\textheight]{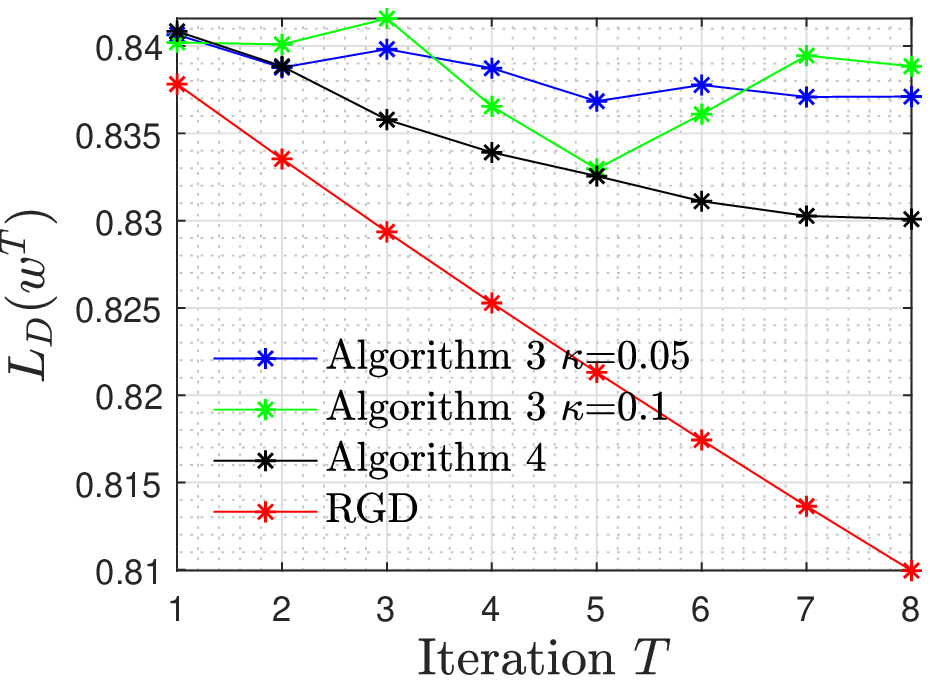}
    \caption{$\epsilon=1$ \label{fig2:c}}
    \end{subfigure}
    ~
    \begin{subfigure}[b]{.24\textwidth}
    \includegraphics[width=\textwidth,height=0.15\textheight]{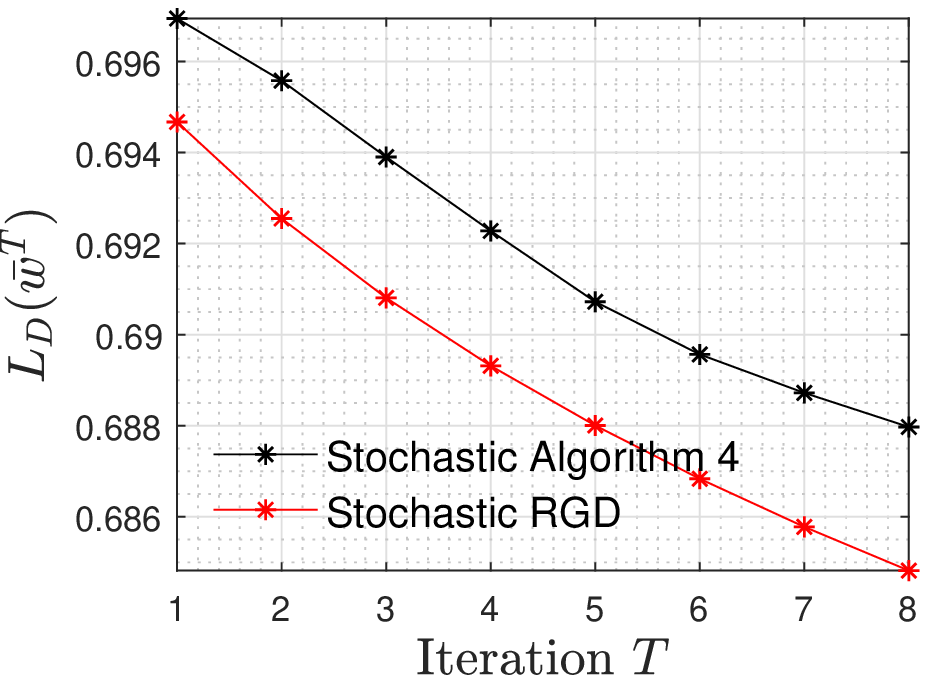}
    \caption{$\epsilon=0.5$ \label{fig2:d}}
    \end{subfigure}    
    
    \caption{Experiments on UCI Adult dataset. Figures \ref{fig2:a} and \ref{fig2:b} are for ridge regressions. Figures \ref{fig2:c} and \ref{fig2:d} are for logistic regressions. \label{fig:2} }
\end{figure*}

\begin{figure*}[!htbp]
    \centering
    \begin{subfigure}[b]{.24\textwidth}
    \includegraphics[width=\textwidth,height=0.15\textheight]{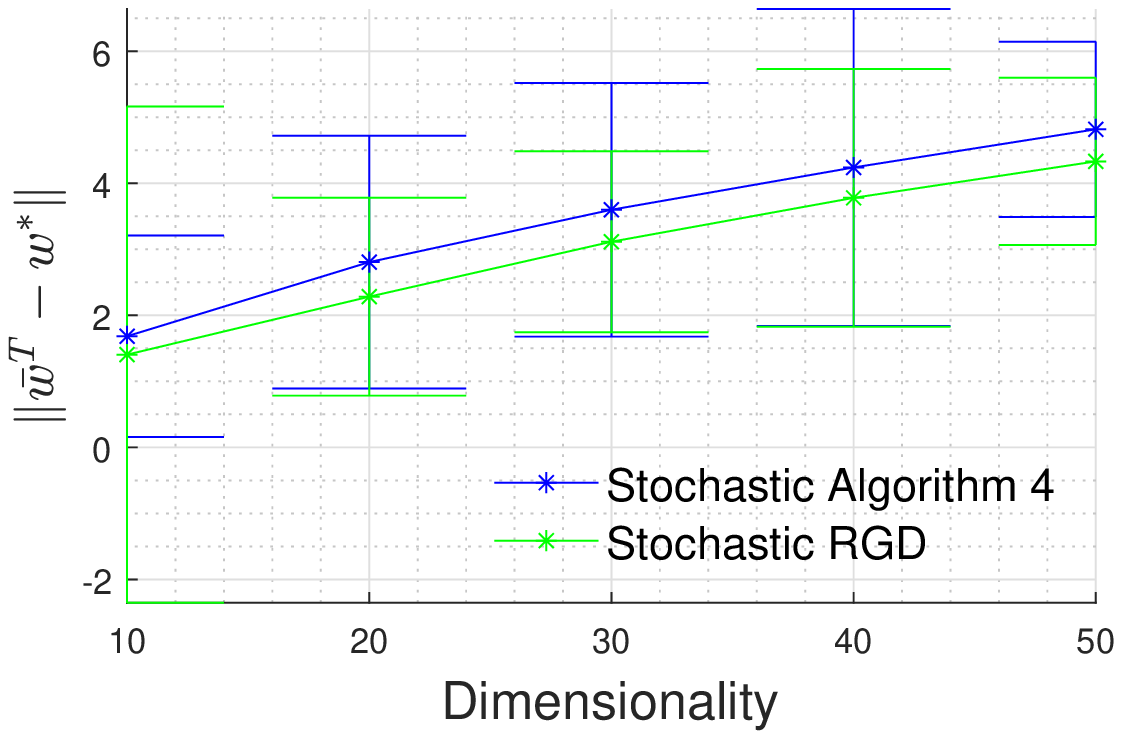}
    \caption{$\epsilon=0.5$ \label{fig3:a}}
    \end{subfigure}
    ~
    \begin{subfigure}[b]{.24\textwidth}
    \includegraphics[width=\textwidth,height=0.15\textheight]{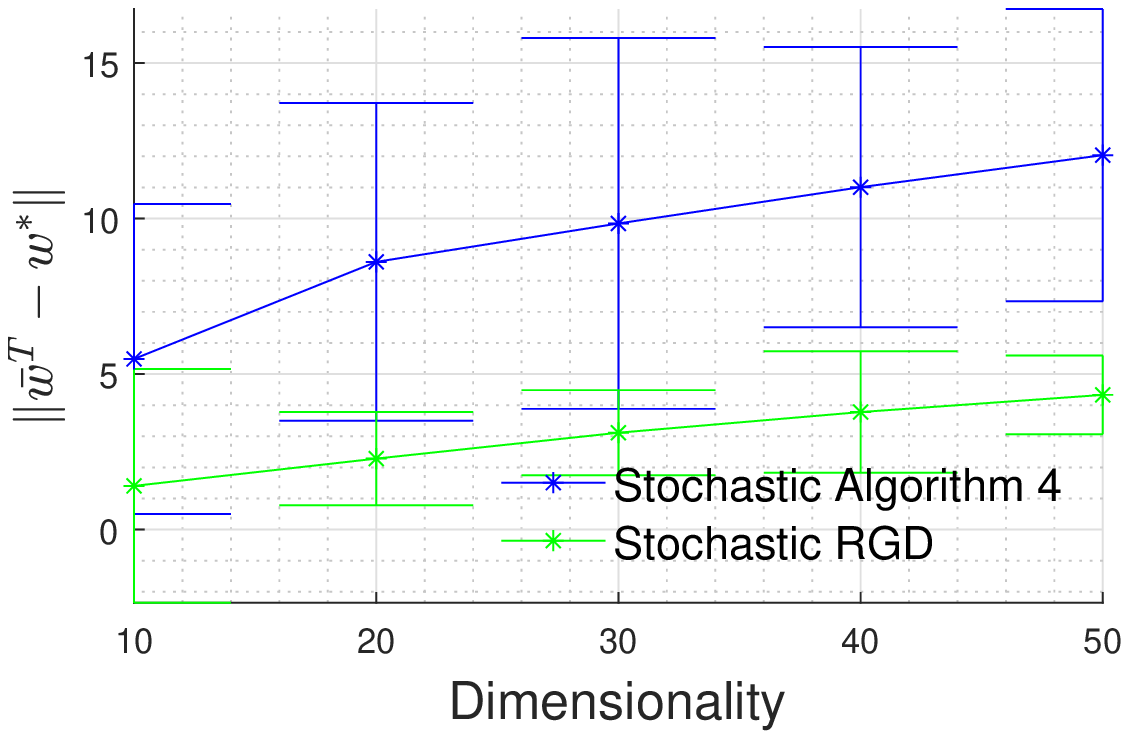}
    \caption{$\epsilon=0.1$ \label{fig3:b}}
    \end{subfigure}    
    \begin{subfigure}[b]{.24\textwidth}
    \includegraphics[width=\textwidth,height=0.15\textheight]{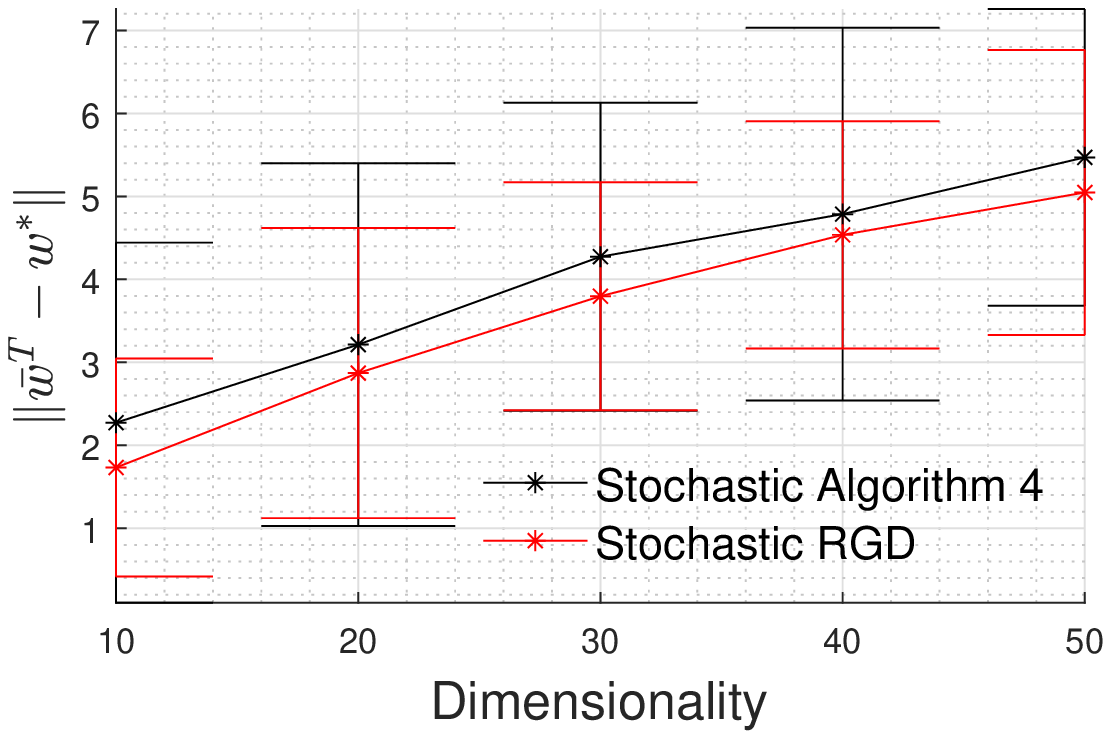}
    \caption{$\epsilon=0.5$ \label{fig3:c}}
    \end{subfigure}   
      \begin{subfigure}[b]{.24\textwidth}
    \includegraphics[width=\textwidth,height=0.15\textheight]{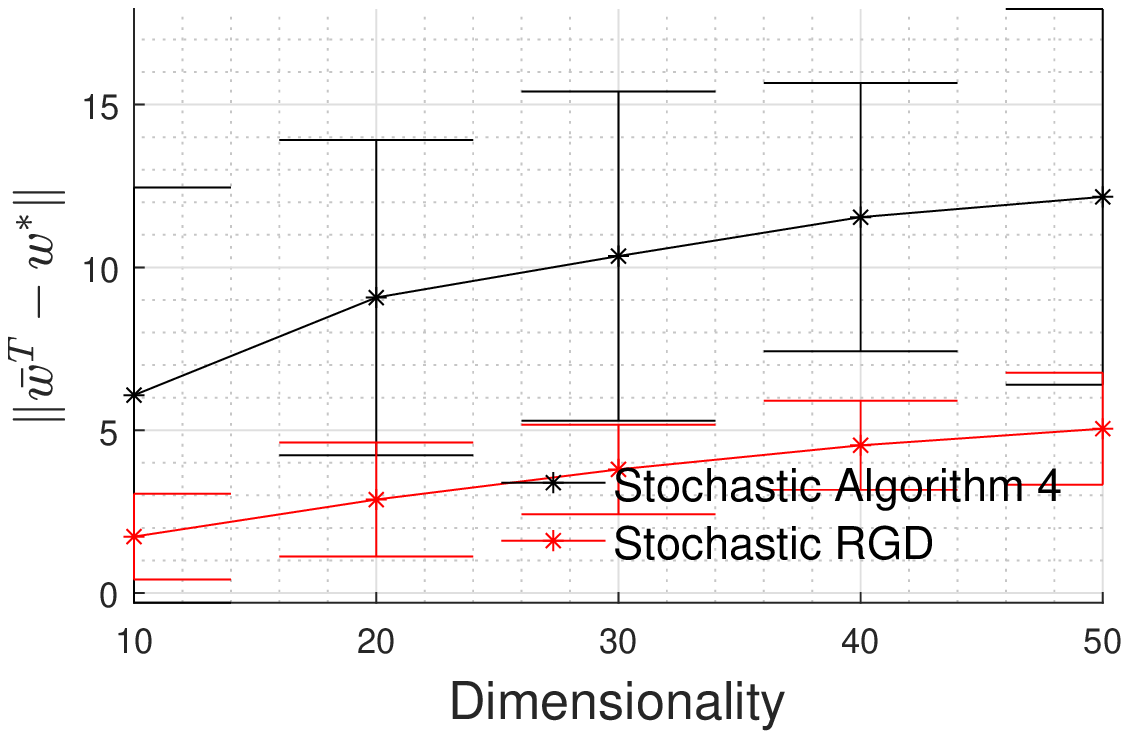}
    \caption{$\epsilon=0.1$ \label{fig3:d}}
    \end{subfigure}     
    \caption{Experiments for the impact of dimensionality. Figure \ref{fig3:a} and \ref{fig3:b} are for ridge regressions.  Figure \ref{fig3:c} and \ref{fig3:d} are for logistic regressions. \label{fig:3} }
\end{figure*}
\begin{figure*}[!htbp]
    \centering
    \begin{subfigure}[b]{.24\textwidth}
    \includegraphics[width=\textwidth,height=0.15\textheight]{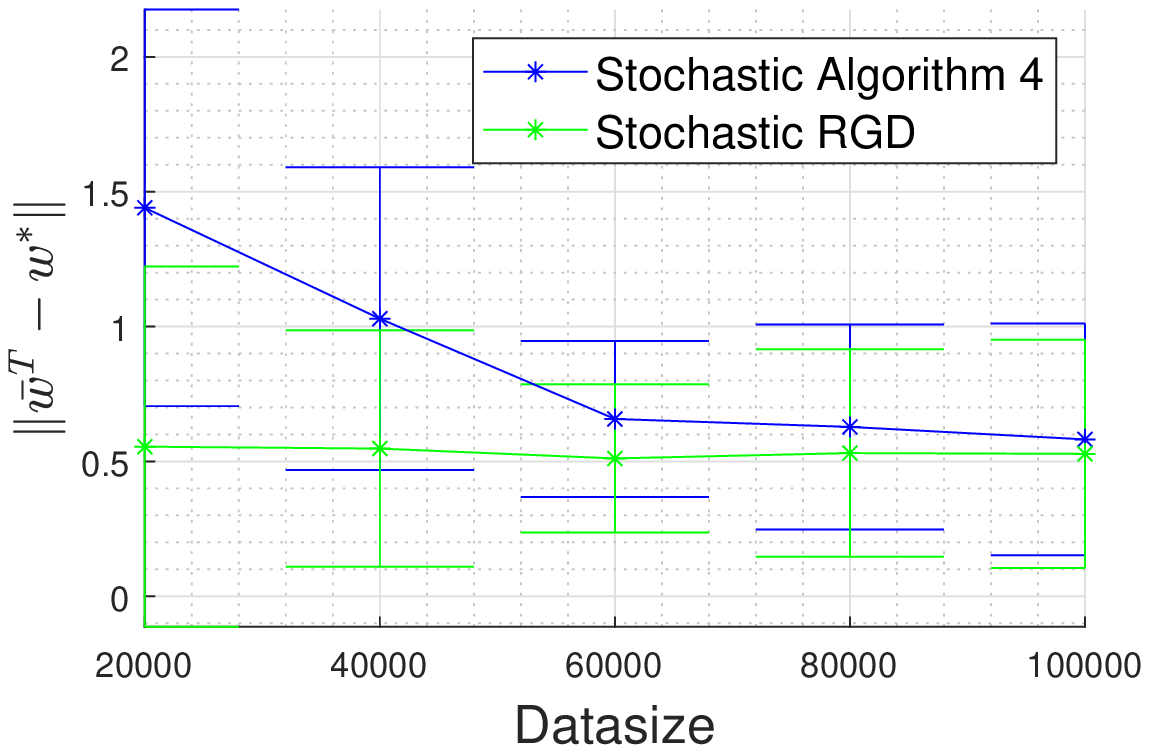}
    \caption{$\epsilon=0.5$ \label{fig4:a}}
    \end{subfigure}
    ~
    \begin{subfigure}[b]{.24\textwidth}
    \includegraphics[width=\textwidth,height=0.15\textheight]{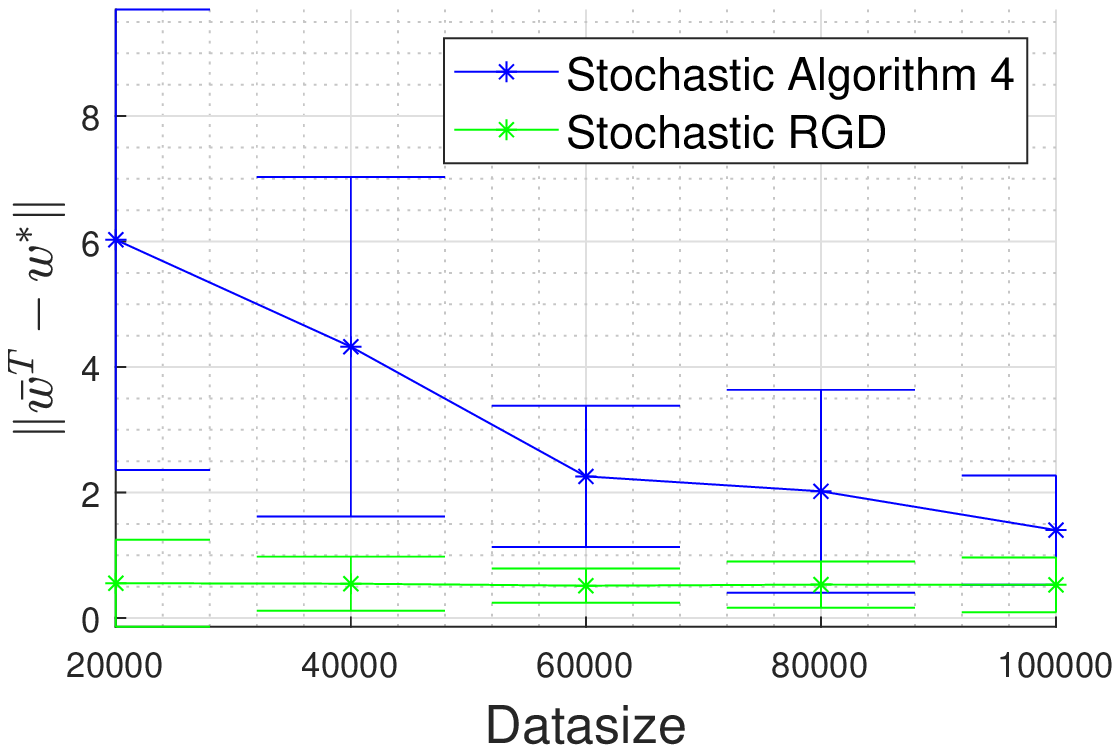}
    \caption{$\epsilon=0.1$ \label{fig4:b}}
    \end{subfigure}    
    \begin{subfigure}[b]{.24\textwidth}
    \includegraphics[width=\textwidth,height=0.15\textheight]{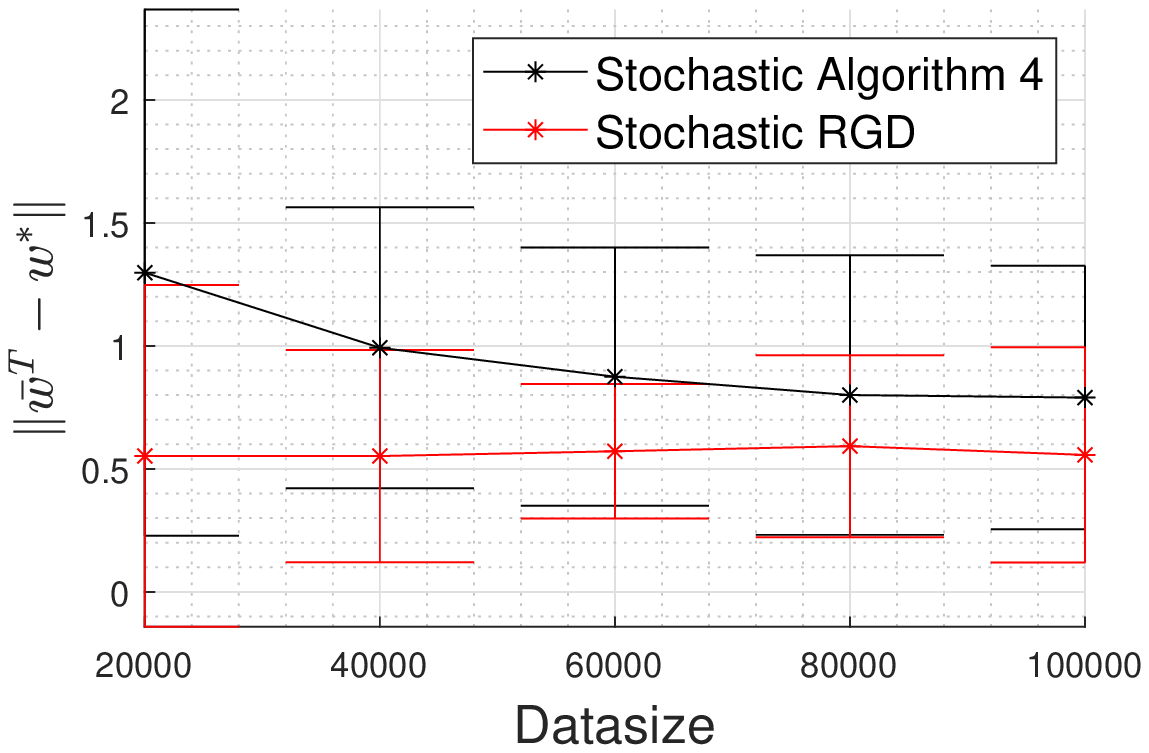}
    \caption{$\epsilon=0.5$ \label{fig4:c}}
    \end{subfigure}   
      \begin{subfigure}[b]{.24\textwidth}
    \includegraphics[width=\textwidth,height=0.15\textheight]{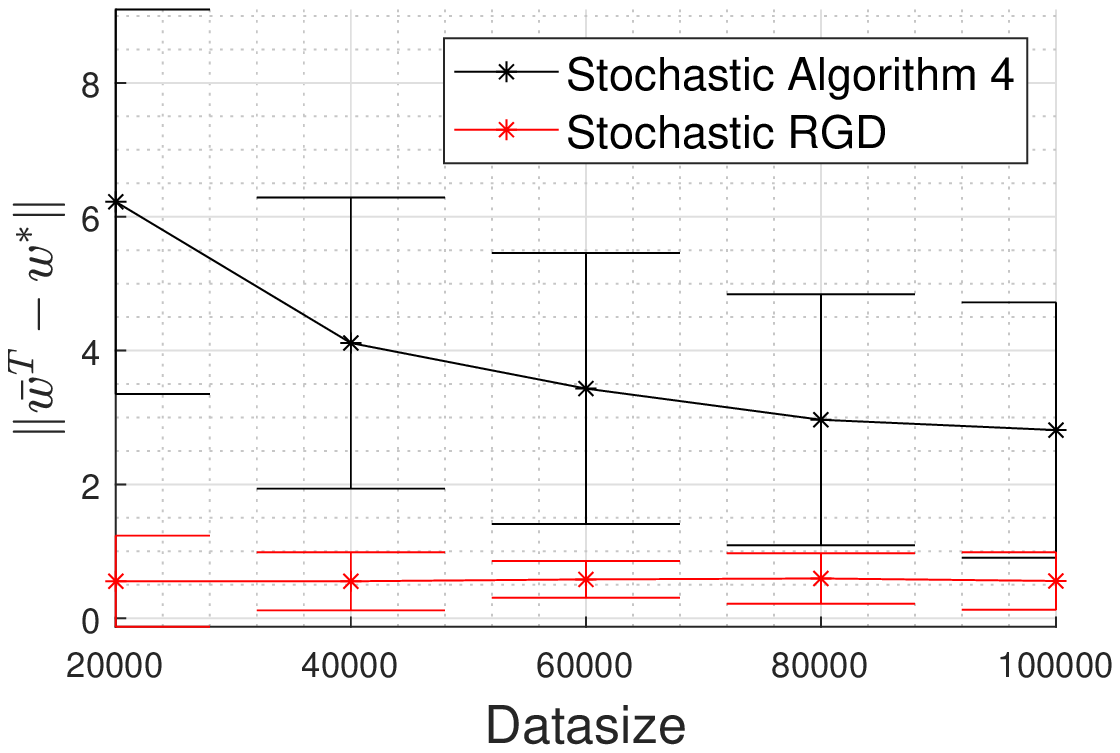}
    \caption{$\epsilon=0.1$ \label{fig4:d}}
    \end{subfigure}     
    \caption{Experiments for the impact of the size of the dataset. Figure \ref{fig4:a} and \ref{fig4:b} are for ridge regressions.  Figure \ref{fig4:c} and \ref{fig4:d} are for logistic regressions.  \label{fig:4}}
\end{figure*}

\section{Discussion}
In this paper, we provide the first comprehensive study on DP-SCO with heavy-tailed data. 
To the best of our knowledge, this is the first work on this problem. 
Specifically, we give a systematic analysis on the problem and design the first efficient algorithms to solve it. In various settings, we bound the (expected) excess generalization risk in both addictive and multiplicative manners. However, 
the problem is far from being closed. First, it is unclear  whether the upper bounds of the excess population risk for strongly convex and general convex loss functions can be further improved.  The second open problem is that we do not know what the lower bound for the excess population risk for these two cases is. Finally, it is an open problem to determine whether we can further relax the assumptions in our previous theorems. We leave these open problems for future research. 

\section*{Acknowledgements}
Di Wang and Jinhui Xu were supported in part by the National Science Foundation (NSF) under Grant No. CCF-1716400 and IIS-1919492.

\bibliography{aistats2020}

\begin{thebibliography}{46}
\providecommand{\natexlab}[1]{#1}
\providecommand{\url}[1]{\texttt{#1}}
\expandafter\ifx\csname urlstyle\endcsname\relax
  \providecommand{\doi}[1]{doi: #1}\else
  \providecommand{\doi}{doi: \begingroup \urlstyle{rm}\Url}\fi

\bibitem[Abadi et~al.(2016)Abadi, Chu, Goodfellow, McMahan, Mironov, Talwar,
  and Zhang]{abadi2016deep}
Abadi, M., Chu, A., Goodfellow, I., McMahan, H.~B., Mironov, I., Talwar, K.,
  and Zhang, L.
\newblock Deep learning with differential privacy.
\newblock In \emph{Proceedings of the 2016 ACM SIGSAC Conference on Computer
  and Communications Security}, pp.\  308--318. ACM, 2016.

\bibitem[Bassily et~al.(2014)Bassily, Smith, and Thakurta]{bassily2014private}
Bassily, R., Smith, A., and Thakurta, A.
\newblock Private empirical risk minimization: Efficient algorithms and tight
  error bounds.
\newblock In \emph{Foundations of Computer Science (FOCS), 2014 IEEE 55th
  Annual Symposium on}, pp.\  464--473. IEEE, 2014.

\bibitem[Bassily et~al.(2019)Bassily, Feldman, Talwar, and
  Thakurta]{bassily2020}
Bassily, R., Feldman, V., Talwar, K., and Thakurta, A.
\newblock Private stochastic convex optimization with optimal rates.
\newblock In \emph{NeurIPS}, 2019.

\bibitem[Biswas et~al.(2007)Biswas, Datta, Fine, and
  Segal]{biswas2007statistical}
Biswas, A., Datta, S., Fine, J.~P., and Segal, M.~R.
\newblock \emph{Statistical advances in the biomedical science}.
\newblock Wiley Online Library, 2007.

\bibitem[Brownlees et~al.(2015)Brownlees, Joly, Lugosi,
  et~al.]{brownlees2015empirical}
Brownlees, C., Joly, E., Lugosi, G., et~al.
\newblock Empirical risk minimization for heavy-tailed losses.
\newblock \emph{The Annals of Statistics}, 43\penalty0 (6):\penalty0
  2507--2536, 2015.

\bibitem[Bubeck et~al.(2015)]{bubeck2015convex}
Bubeck, S. et~al.
\newblock Convex optimization: Algorithms and complexity.
\newblock \emph{Foundations and Trends{\textregistered} in Machine Learning},
  8\penalty0 (3-4):\penalty0 231--357, 2015.

\bibitem[Bun \& Steinke(2016)Bun and Steinke]{bun2016concentrated}
Bun, M. and Steinke, T.
\newblock Concentrated differential privacy: Simplifications, extensions, and
  lower bounds.
\newblock In \emph{Theory of Cryptography Conference}, pp.\  635--658.
  Springer, 2016.

\bibitem[Bun \& Steinke(2019)Bun and Steinke]{bun2019average}
Bun, M. and Steinke, T.
\newblock Average-case averages: Private algorithms for smooth sensitivity and
  mean estimation.
\newblock \emph{arXiv preprint arXiv:1906.02830}, 2019.

\bibitem[Catoni \& Giulini(2017)Catoni and Giulini]{catoni2017dimension}
Catoni, O. and Giulini, I.
\newblock Dimension-free pac-bayesian bounds for matrices, vectors, and linear
  least squares regression.
\newblock \emph{arXiv preprint arXiv:1712.02747}, 2017.

\bibitem[Chaudhuri \& Monteleoni(2009)Chaudhuri and
  Monteleoni]{chaudhuri2009privacy}
Chaudhuri, K. and Monteleoni, C.
\newblock Privacy-preserving logistic regression.
\newblock In \emph{Advances in neural information processing systems}, pp.\
  289--296, 2009.

\bibitem[Chaudhuri et~al.(2011)Chaudhuri, Monteleoni, and
  Sarwate]{chaudhuri2011differentially}
Chaudhuri, K., Monteleoni, C., and Sarwate, A.~D.
\newblock Differentially private empirical risk minimization.
\newblock \emph{Journal of Machine Learning Research}, 12\penalty0
  (Mar):\penalty0 1069--1109, 2011.

\bibitem[Chen et~al.(2017)Chen, Su, and Xu]{chen2017distributed}
Chen, Y., Su, L., and Xu, J.
\newblock Distributed statistical machine learning in adversarial settings:
  Byzantine gradient descent.
\newblock \emph{Proceedings of the ACM on Measurement and Analysis of Computing
  Systems}, 1\penalty0 (2):\penalty0 44, 2017.

\bibitem[Ding et~al.(2017)Ding, Kulkarni, and Yekhanin]{ding2017collecting}
Ding, B., Kulkarni, J., and Yekhanin, S.
\newblock Collecting telemetry data privately.
\newblock In \emph{Advances in Neural Information Processing Systems}, pp.\
  3571--3580, 2017.

\bibitem[Dua \& Graff(2017)Dua and Graff]{Dua:2019}
Dua, D. and Graff, C.
\newblock {UCI} machine learning repository, 2017.
\newblock URL \url{http://archive.ics.uci.edu/ml}.

\bibitem[Duchi et~al.(2013)Duchi, Jordan, and Wainwright]{duchi2013local}
Duchi, J.~C., Jordan, M.~I., and Wainwright, M.~J.
\newblock Local privacy and statistical minimax rates.
\newblock In \emph{2013 IEEE 54th Annual Symposium on Foundations of Computer
  Science}, pp.\  429--438. IEEE, 2013.

\bibitem[Dwork \& Lei(2009)Dwork and Lei]{dwork2009differential}
Dwork, C. and Lei, J.
\newblock Differential privacy and robust statistics.
\newblock In \emph{Proceedings of the forty-first annual ACM symposium on
  Theory of computing}, pp.\  371--380. ACM, 2009.

\bibitem[Dwork et~al.(2006)Dwork, McSherry, Nissim, and
  Smith]{dwork2006calibrating}
Dwork, C., McSherry, F., Nissim, K., and Smith, A.
\newblock Calibrating noise to sensitivity in private data analysis.
\newblock In \emph{Theory of cryptography conference}, pp.\  265--284.
  Springer, 2006.

\bibitem[Fama(1963)]{fama1963mandelbrot}
Fama, E.~F.
\newblock Mandelbrot and the stable paretian hypothesis.
\newblock \emph{The journal of business}, 36\penalty0 (4):\penalty0 420--429,
  1963.

\bibitem[Feldman \& Steinke(2018)Feldman and Steinke]{feldman2018calibrating}
Feldman, V. and Steinke, T.
\newblock Calibrating noise to variance in adaptive data analysis.
\newblock In \emph{Conference On Learning Theory}, pp.\  535--544, 2018.

\bibitem[Holland(2019)]{holland2019a}
Holland, M.~J.
\newblock Robust descent using smoothed multiplicative noise.
\newblock In \emph{22nd International Conference on Artificial Intelligence and
  Statistics (AISTATS)}, volume~89 of \emph{Proceedings of Machine Learning
  Research}, pp.\  703--711, 2019.

\bibitem[Holland \& Ikeda(2017)Holland and Ikeda]{holland2017efficient}
Holland, M.~J. and Ikeda, K.
\newblock Efficient learning with robust gradient descent.
\newblock \emph{Machine Learning}, pp.\  1--38, 2017.

\bibitem[Hsu \& Sabato(2016)Hsu and Sabato]{hsu2016loss}
Hsu, D. and Sabato, S.
\newblock Loss minimization and parameter estimation with heavy tails.
\newblock \emph{The Journal of Machine Learning Research}, 17\penalty0
  (1):\penalty0 543--582, 2016.

\bibitem[Ibragimov et~al.(2015)Ibragimov, Ibragimov, and
  Walden]{ibragimov2015heavy}
Ibragimov, M., Ibragimov, R., and Walden, J.
\newblock \emph{Heavy-tailed distributions and robustness in economics and
  finance}, volume 214.
\newblock Springer, 2015.

\bibitem[Juditsky \& Nemirovski(2008)Juditsky and
  Nemirovski]{juditsky2008large}
Juditsky, A. and Nemirovski, A.~S.
\newblock Large deviations of vector-valued martingales in 2-smooth normed
  spaces.
\newblock \emph{arXiv preprint arXiv:0809.0813}, 2008.

\bibitem[Karwa \& Vadhan(2017)Karwa and Vadhan]{karwa2017finite}
Karwa, V. and Vadhan, S.
\newblock Finite sample differentially private confidence intervals.
\newblock \emph{arXiv preprint arXiv:1711.03908}, 2017.

\bibitem[Kasiviswanathan \& Jin(2016)Kasiviswanathan and
  Jin]{kasiviswanathan2016efficient}
Kasiviswanathan, S.~P. and Jin, H.
\newblock Efficient private empirical risk minimization for high-dimensional
  learning.
\newblock In \emph{International Conference on Machine Learning}, pp.\
  488--497, 2016.

\bibitem[Kifer et~al.(2012)Kifer, Smith, and Thakurta]{kifer2012private}
Kifer, D., Smith, A., and Thakurta, A.
\newblock Private convex empirical risk minimization and high-dimensional
  regression.
\newblock In \emph{Conference on Learning Theory}, pp.\  25--1, 2012.

\bibitem[Lecu{\'e} et~al.(2018)Lecu{\'e}, Lerasle, and
  Mathieu]{lecue2018robust}
Lecu{\'e}, G., Lerasle, M., and Mathieu, T.
\newblock Robust classification via mom minimization.
\newblock \emph{arXiv preprint arXiv:1808.03106}, 2018.

\bibitem[Lorentz(1966)]{lorentz1966metric}
Lorentz, G.
\newblock Metric entropy and approximation.
\newblock \emph{Bulletin of the American Mathematical Society}, 72\penalty0
  (6):\penalty0 903--937, 1966.

\bibitem[Mandelbrot(1997)]{mandelbrot1997variation}
Mandelbrot, B.~B.
\newblock The variation of certain speculative prices.
\newblock In \emph{Fractals and scaling in finance}, pp.\  371--418. Springer,
  1997.

\bibitem[Minsker et~al.(2015)]{minsker2015geometric}
Minsker, S. et~al.
\newblock Geometric median and robust estimation in banach spaces.
\newblock \emph{Bernoulli}, 21\penalty0 (4):\penalty0 2308--2335, 2015.

\bibitem[Nissim et~al.(2007)Nissim, Raskhodnikova, and Smith]{nissim2007smooth}
Nissim, K., Raskhodnikova, S., and Smith, A.
\newblock Smooth sensitivity and sampling in private data analysis.
\newblock In \emph{Proceedings of the thirty-ninth annual ACM symposium on
  Theory of computing}, pp.\  75--84. ACM, 2007.

\bibitem[Smith et~al.(2017)Smith, Thakurta, and Upadhyay]{smith2017interaction}
Smith, A., Thakurta, A., and Upadhyay, J.
\newblock Is interaction necessary for distributed private learning?
\newblock In \emph{2017 IEEE Symposium on Security and Privacy (SP)}, pp.\
  58--77. IEEE, 2017.

\bibitem[Srebro et~al.(2010)Srebro, Sridharan, and
  Tewari]{srebro2010smoothness}
Srebro, N., Sridharan, K., and Tewari, A.
\newblock Smoothness, low noise and fast rates.
\newblock In \emph{Advances in neural information processing systems}, pp.\
  2199--2207, 2010.

\bibitem[Su et~al.(2016)Su, Cao, Li, Bertino, and Jin]{su2016differentially}
Su, D., Cao, J., Li, N., Bertino, E., and Jin, H.
\newblock Differentially private k-means clustering.
\newblock In \emph{Proceedings of the sixth ACM conference on data and
  application security and privacy}, pp.\  26--37. ACM, 2016.

\bibitem[Talwar et~al.(2015)Talwar, Thakurta, and Zhang]{talwar2015nearly}
Talwar, K., Thakurta, A.~G., and Zhang, L.
\newblock Nearly optimal private lasso.
\newblock In \emph{Advances in Neural Information Processing Systems}, pp.\
  3025--3033, 2015.

\bibitem[Tang et~al.(2017)Tang, Korolova, Bai, Wang, and Wang]{apple}
Tang, J., Korolova, A., Bai, X., Wang, X., and Wang, X.
\newblock Privacy loss in apple's implementation of differential privacy on
  macos 10.12.
\newblock \emph{CoRR}, abs/1709.02753, 2017.

\bibitem[Vapnik(2013)]{vapnik2013nature}
Vapnik, V.
\newblock \emph{The nature of statistical learning theory}.
\newblock Springer science \& business media, 2013.

\bibitem[Vershynin(2010)]{vershynin2010introduction}
Vershynin, R.
\newblock Introduction to the non-asymptotic analysis of random matrices.
\newblock \emph{arXiv preprint arXiv:1011.3027}, 2010.

\bibitem[Wang et~al.(2017)Wang, Ye, and Xu]{wang2017differentially}
Wang, D., Ye, M., and Xu, J.
\newblock Differentially private empirical risk minimization revisited: Faster
  and more general.
\newblock In \emph{Advances in Neural Information Processing Systems}, pp.\
  2722--2731, 2017.

\bibitem[Wang et~al.(2018)Wang, Gaboardi, and Xu]{wang2018empirical}
Wang, D., Gaboardi, M., and Xu, J.
\newblock Empirical risk minimization in non-interactive local differential
  privacy revisited.
\newblock In \emph{Advances in Neural Information Processing Systems}, pp.\
  965--974, 2018.

\bibitem[Wang et~al.(2019{\natexlab{a}})Wang, Chen, and
  Xu]{wang2019differentially}
Wang, D., Chen, C., and Xu, J.
\newblock Differentially private empirical risk minimization with non-convex
  loss functions.
\newblock In \emph{International Conference on Machine Learning}, pp.\
  6526--6535, 2019{\natexlab{a}}.

\bibitem[Wang et~al.(2019{\natexlab{b}})Wang, Smith, and
  Xu]{wang2019noninteractive}
Wang, D., Smith, A., and Xu, J.
\newblock Noninteractive locally private learning of linear models via
  polynomial approximations.
\newblock In \emph{Algorithmic Learning Theory}, pp.\  897--902,
  2019{\natexlab{b}}.

\bibitem[Wang et~al.(2015)Wang, Wang, and Singh]{wang2015differentially}
Wang, Y., Wang, Y.-X., and Singh, A.
\newblock Differentially private subspace clustering.
\newblock In \emph{Advances in Neural Information Processing Systems}, pp.\
  1000--1008, 2015.

\bibitem[Woolson \& Clarke(2011)Woolson and Clarke]{woolson2011statistical}
Woolson, R.~F. and Clarke, W.~R.
\newblock \emph{Statistical methods for the analysis of biomedical data},
  volume 371.
\newblock John Wiley \& Sons, 2011.

\bibitem[Wu et~al.(2017)Wu, Li, Kumar, Chaudhuri, Jha, and
  Naughton]{wu2017bolt}
Wu, X., Li, F., Kumar, A., Chaudhuri, K., Jha, S., and Naughton, J.
\newblock Bolt-on differential privacy for scalable stochastic gradient
  descent-based analytics.
\newblock In \emph{Proceedings of the 2017 ACM International Conference on
  Management of Data}, pp.\  1307--1322. ACM, 2017.

\end{thebibliography}
\bibliographystyle{icml2020}

\appendix
\section{Omitted Proofs}
\begin{proof}[{\bf Proof of Lemma 1}]
Before the proof, we recall the following two lemmas 
\begin{lemma}[\citep{srebro2010smoothness}]\label{lemma:1}
If a non-negative function $f:\mathcal{W}\mapsto \mathbb{R}_+$ is $\beta$-smooth, then $\|\nabla f(w)\|_2^2\leq 4\beta f(w)$ for all $w\in \mathcal{W}$. 
\end{lemma}subscribe
\begin{lemma}[\citep{juditsky2008large}]\label{lemma:2}
Let $X_1, X_2, \cdots, X_n$ be independent copies of a zero-mean random vector $X$, then $\mathbb{E}\|\frac{1}{n}\sum_{i=1}^nX_i\|_2^2\leq \frac{1}{n}\mathbb{E}\|X\|_2^2$.  
\end{lemma}

Consider $w=w^*$. Then by Assumption 1, we have $\nabla L(w^*)= \mathbb{E}[\nabla \ell(w^*, x)]=0$. Thus, by Lemma 2 we have
\begin{equation*}
   \mathbb{E} \|\nabla \hat{L}(w^*, D)\|^2_2\leq \frac{1}{n}\mathbb{E}[\|\nabla \ell(w^*, x)\|_2^2].
\end{equation*}
By Markov's inequality, we get 
\begin{equation*}
    \text{Pr}[\|\nabla \hat{L}(w^*, D)\|^2_2\leq \frac{10}{n}\mathbb{E}[\|\nabla \ell(w^*, x)\|_2^2]\geq \frac{9}{10}. 
\end{equation*}
Since $n\geq n_\alpha$, by the assumption we have with probability at least $\frac{5}{6}$ that $\hat{L}(w, D)$ is $\alpha$ strongly convex. Thus, we get 
\begin{align*}
    &\frac{\alpha}{2}\|w_D-w^*\|_2^2\leq \\
    &-\langle \nabla \hat{L}(w^*, D), w_D-w^*\rangle +\hat{L}(w_D, D)-\hat{L}(w^*,D)\\
    &\leq \|\nabla \hat{L}(w^*, D)\|_2 \|w_D-w^*\|_2. 
\end{align*}
In total, with probability at least $\frac{3}{4} $, we have
\begin{equation*}
    \|w_D-w^*\|_2\leq \sqrt{\frac{40\mathbb{E}\|\nabla \ell(w^*, x)\|_2^2 }{n\alpha^2}}. 
\end{equation*}
\end{proof}

\begin{proof}[{\bf Proof of Theorem 2}]
For each subsample set $D_{S_i}$, by the assumption we have its size $\frac{n}{m}\geq n_\alpha$. Thus, Lemma 1 holds with $n=\frac{n}{m}$. That is,  (1) holds with $r=\sqrt{\frac{40m\mathbb{E}\|\nabla \ell(w^*, x)\|_2^2 }{n\alpha^2}} $. Hence, by Theorem 1 we have 
\begin{equation*}
    \|\mathcal{A}(D)-w^*\|_2 \leq O(\frac{\sqrt{d}r}{\epsilon})= O(\sqrt{\frac{dm\mathbb{E}\|\nabla \ell(w^*, x)\|_2^2 }{n\epsilon^2\alpha^2}}).
\end{equation*}
Since $L_\mathcal{D} (w)$ is $\beta$-smooth and $\nabla L_\mathcal{D}(w^*)=0$, we have 
$L_\mathcal{D} (\mathcal{A}(D))-L_\mathcal{D} (w^*)\leq \frac{\beta}{2}\|\mathcal{A}(D)-w^*\|_2^2$. Also, by Lemma 1 and the non-negative property we get 
\begin{align*}
    L_\mathcal{D} (\mathcal{A}(D))-L_\mathcal{D} (w^*)\leq O( (\frac{\beta}{\alpha})^2 \frac{dm}{n\epsilon^2}L_\mathcal{D}(w^*)). 
\end{align*}
Taking $m = \tilde{\Theta}(\frac{d^2}{\epsilon^2})$, we get the proof. 
\end{proof}

\begin{proof}[{\bf Proof of Theorem 4}]
We first give the definition of zCDP in \citep{bun2016concentrated}.

\begin{definition}
    A randomized algorithm $\mathcal{A}: \mathcal{X}^n\mapsto \mathcal{Y}$ is $\rho$-zero Concentrated Differentially Private (zCDP) if for all neighboring datasets $D\sim D'$ and all $\alpha\in (1, \infty)$, 
    \begin{equation*}
        D_\alpha (\mathcal{A}(D)\| \mathcal{A}(D'))\leq \rho\alpha, 
    \end{equation*}
    where $ D_\alpha (P\| Q)=\frac{1}{\alpha-1}\log \mathbb{E}_{X\sim P}[(\frac{P(X)}{Q(X)})^{\alpha-1}]$ denotes the R\'{e}nyi divergence of order $\alpha$.
\end{definition}

We first convert $(\epsilon, \delta)$-DP to $\frac{1}{2}\tilde{\epsilon}^2$-zCDP by using the following lemma 
\begin{lemma}[\citep{bun2016concentrated}]\label{alemma:5}
Let $M: \mathcal{X}^n\mapsto \mathcal{Y}$ be a randomized algorithm. If $M$ is $\frac{1}{2}\epsilon^2$-zCDP,  it is 
$(\frac{1}{2}\epsilon^2+\epsilon\cdot \sqrt{2\log \frac{1}{\delta}}, \delta)$-DP for all $\delta>0$. 
\end{lemma}

Thus, it  suffices to show that 
Algorithm 3 is $\frac{1}{2}\tilde{\epsilon}^2$-zCDP. We note that in each iteration and each coordinate, outputting $\nabla_{t-1, j}$ will be $\frac{1}{2}\frac{\tilde{\epsilon}^2}{dT}$-zCDP by Theorem 3. Thus by the composition property of CDP, we know that 
it is  $\frac{1}{2}\tilde{\epsilon^2}$-zCDP.
\end{proof}

\begin{proof}[{\bf Proof of Lemma 2}]
By assumption, we know that $\mathcal{W}$ is closed and bounded, and hence it is compact. By \citep{lorentz1966metric} we know that its covering number with radius $\delta$ (will be specified later) is bounded from above as $N_\delta\leq (\frac{3\Delta}{2\delta})^d$. Denote the center of this $\delta$-net as $\tilde{\mathcal{W}}= \{\tilde{w}_1, \tilde{w}_2, \cdots, \tilde{w}_{N_\delta}\}$. 

We first fix $j\in [d]$ and consider $|\tilde{\nabla}_{j}(w)-\nabla_j L_{D}(w)|$ (we omit the subscript $t-1$). Then, we have 
\begin{align}
   &\mathbb{E}_{Z_j}(\tilde{\nabla}_{j}(w)-\nabla_j L_{D}(w))^2= \nonumber \\ &\mathbb{E}\big([\text{Trim}_{m}(D_j(w))]_{[a,b]}+ \frac{1}{s} S^{t}_{[trim(\cdot)]_{[a, b]}}(D_j(w))\cdot Z_j \nonumber
   \\& - \nabla_j L_{D}(w)\big)^2 \nonumber \\
   &\leq O(([\text{Trim}_{m}(D_j(w))]_{[a,b]}- \nabla_j L_{D}(w))^2 \nonumber\\
   &+\mathbb{E} ( \frac{1}{s} S^{t}_{[trim(\cdot)]_{[a, b]}}(D_j(w))\cdot Z_j)^2 )\nonumber \\
   &\leq O((\text{Trim}_{m}(D_j(w))]- \nabla_j L_{D}(w))^2 \nonumber \\
   &+\mathbb{E} ( \frac{1}{s} S^{t}_{[trim_m(\cdot)]_{[a, b]}}(D(w))\cdot Z_j)^2), \label{aeq:1}
\end{align}
where $D_j(w)=\{\nabla_j \ell(w, x_i)\}_{i=1}^n$ and the last inequality is due to the property that the truncation operation reduces error.

\begin{lemma}\label{alemma:6}
Let $a\leq \mu\leq b$ and  $X$ be a random variable. Then 
\begin{equation*}
   ([X]_{[a, b]}-\mu)^2\leq (x-\mu)^2. 
\end{equation*}
\end{lemma}
By the proof of Theorem 51 in \citep{bun2019average} and the fact that $\epsilon=\frac{\tilde{\epsilon}}{\sqrt{dT}}$, we have ($m, a, b=O(1)$)
\begin{equation}\label{aeq:0}
    \mathbb{E}_Z ( \frac{1}{s} S^{t}_{[trim_m(\cdot)]_{[a, b]}}(D_j(w))\cdot Z)^2\leq O(\frac{\tau^2 dT\log n}{n\tilde{\epsilon}^2}),
\end{equation}
where the $O$-notation omits the $\log \sigma^2$ and $\log (b-a)$ factors. 

Next, we  bound the first term of (\ref{aeq:1}). Before showing that, we first give the following estimation error on the trimming operation for sub-exponential random variables.

\begin{lemma}\label{alemma:7}
Suppose that $x_i$ are i.i.d $\upsilon$-sub-exponential with mean $\mu$. Then, the following holds for any $t\geq 0$, 
\begin{equation*}
    \mathbb{P}\{\frac{1}{n}\sum_{i=1}^n x_i-\mu \geq t\}\leq 2\exp(-n\min\{\frac{t}{2v}, \frac{t^2}{2v^2}\}),
    \end{equation*}
    and for any $s\geq 0$, 
    \begin{equation*}
        \mathbb{P}[\max_{i\in [n]}\{|x_i-\mu|\}\geq s ] \leq 2n\exp(-\min\{\frac{s}{2v}, \frac{s^2}{2v^2}\}),
    \end{equation*}
    and for any $m\geq 0$, under the above two events, 
    \begin{equation*}
        |\text{Trim}_{m}(\{x_i\}_{i=1}^n)-\mu|\leq \frac{nt+ms}{n-2m}. 
    \end{equation*}
\end{lemma}

\begin{proof}[Proof of Lemma \ref{alemma:7}]
Note that the first two inequalities are just the Berstein's Inequality. We only prove the last inequality. 

Let $\mathcal{T}\subset [n]$ denote the set of all trimmed variables and $\mathcal{U}=[n]\backslash \mathcal{T}$. Then, we know that  $\text{Trim}_{m}(\{x_i\}_{i=1}^n)=\frac{\sum_{i\in \mathcal{U}}x_i}{n-2m}$. Thus, we have 
\begin{align}
   & |\frac{\sum_{i\in \mathcal{U}}x_i}{n-2m}-\mu|= \frac{1}{n-2m}|\sum_{i\in [n]}(x_i-\mu)-\sum_{i\in \mathcal{T}}(x_i-\mu)| \nonumber \\ 
    &\leq \frac{1}{n-2m}(|\sum_{i\in [n]}(x_i-\mu)|+ |\sum_{i\in \mathcal{T}}(x_i-\mu)|). \label{aeq:2}
\end{align}
For the second term of (\ref{aeq:2}), we have $|\sum_{i\in \mathcal{T}}(x_i-\mu)|\leq m \max\{|x_i-\mu|\}$.
Plugging the inequalities into (\ref{aeq:2}) we get the proof. 
\end{proof}

Now, fix any $w\in \mathcal{W}$, we know that there exists a $\tilde{w}$
which is in the $\delta$-net, {\em i.e.,} $\|\tilde{w}-w\|_2\leq \delta$. Then by using the Bernstein inequality and the sub-exponential assumption and taking the union bound, we can see that with probability at least $1-2dN_{\delta}\exp(-n\min\{\frac{t}{2\tau }, \frac{t^2}{2\tau^2}\})$, we have the following for  all $j\in [d]$ and $\tilde{w}\in \tilde{\mathcal{W}}$
\begin{equation}\label{aeq:3}
    |\sum_{i=1}^n\frac{\nabla_j \ell(\tilde{w}, x_i)}{n}- \nabla_j L_\mathcal{D} (\tilde{w})| \leq t,
\end{equation}
and with probability at least $1-2dnN_{\delta}\exp(-\min\{\frac{s}{2\tau }, \frac{s^2}{2\tau^2}\})$, we get the following for all $j\in [d]$ and $\tilde{w}\in \tilde{W}$,
\begin{equation}\label{aeq:4}
   \max_{i\in [n]} |{\nabla_j \ell(\tilde{w}, x_i)}- \nabla_j L_\mathcal{D} (\tilde{w})| \leq s.
\end{equation}
By the $\beta_j$-smoothness of $\ell_j(\cdot, x)$ we have 
\begin{equation}
   | \sum_{i=1}^n\frac{\nabla_j \ell(\tilde{w}, x_i)}{n}-\sum_{i=1}^n\frac{\nabla_j \ell(w, x_i)}{n}|\leq \beta_j\|w-\tilde{w}\|_2\leq \beta_j \delta, 
\end{equation}
\begin{equation}
     |\nabla_j L_\mathcal{D} (\tilde{w})- \nabla_j L_\mathcal{D} (w)|\leq \beta_j \delta. 
\end{equation}
Thus, we get 
\begin{align}
    & |\sum_{i=1}^n\frac{\nabla_j \ell(w, x_i)}{n}- \nabla_j L_\mathcal{D} (w)| \leq t+2 \beta_j \delta \label{aeq:5} \\
    & \max_{i\in [n]} |{\nabla_j \ell(w, x_i)}- \nabla_j L_\mathcal{D} (w)| \leq s+2 \beta_j \delta. \label{aeq:6}
\end{align}
By Lemma \ref{alemma:7} we have for all $j\in [d]$ and $w\in \mathcal{W}$
\begin{equation*}
    |\text{Trim}_m (D_j(w))-\nabla_j L_\mathcal{D} (w)|\leq \frac{nt+ms}{n-2m} + \frac{m+n}{n-2m} 2\beta_j\delta.
\end{equation*}
Combining this with (\ref{aeq:0}) we have the following with probability at least $1-2dnN_{\delta}\exp(-\min\{\frac{s}{2\tau }, \frac{s^2}{2\tau^2}\})-2dN_{\delta}\exp(-n\min\{\frac{t}{2\tau }, \frac{t^2}{2\tau^2}\})$ for all $j\in [d]$ and $\tilde{w}\in \tilde{\mathcal{W}}$,
\begin{align}
   & \mathbb{E}\| \nabla \tilde{L}(w, D)- \nabla L_\mathcal{D}(w)\|_2\leq \nonumber \\
   &\leq O(\sqrt{d}\frac{nt+ms}{n-2m}+ \hat{\beta}\delta \frac{m+n}{n-2m}+\frac{\tau d\sqrt{T\log n}}{\sqrt{n}\tilde{\epsilon}}),
\end{align}
where $\hat{\beta}=\sqrt{\beta_1^2+\cdots+ \beta_d^2}$. Thus, let $\delta=\frac{1}{n\hat{\beta}}, m=O(1)$, 
\begin{equation*}
    t= O( \tau \max\{\frac{d}{n}\log(n\hat{\beta}\Delta), \sqrt{\frac{d}{n}\log(n\hat{\beta}\Delta)}\}), 
\end{equation*}
\begin{equation*}
    s= O(\tau d\log(\hat{\beta}n\Delta)). 
\end{equation*}
Then, we get the proof. 
\end{proof}

\begin{proof}[{ \bf Proof of Theorem 5}]
In the $t$-th iteration, let 
\begin{equation*}
    \hat{w}^{t}=w^{t-1}-\eta \nabla \tilde{L}(w^{t-1}, D). 
\end{equation*}
Then, by the property of Euclidean project we have 
\begin{equation*}
    \|w^{t}-w^{t-1}\|_2\leq \|\hat{w}^t-w^{t-1}\|_2. 
\end{equation*}
Hence, we have 
\begin{align*}
    &\|\hat{w}^t-w^*\|_2\leq \|w^{t-1}-\eta \nabla \tilde{L}(w^{t-1}, D)-w^*\|_2 \\
    &\leq \|w^{t-1}-\eta \nabla L_\mathcal{D} (w^{t-1})-w^*\|_2\\
    &+\eta \| \nabla \tilde{L}(w^{t-1}, D)- L_\mathcal{D} (w^{t-1})\|_2. 
\end{align*}
For the first term, by the co-coercivity of strongly convex functions \citep{bubeck2015convex}, we have 
\begin{multline*}
    \langle w^{t-1}-w^*, \nabla L_\mathcal{D} (w^{t-1}) \rangle \geq \frac{\alpha \beta}{\alpha+\beta}\|w^{t-1}-w^*\|_2^2 \\ +\frac{1}{\alpha+\beta}\|\nabla L_\mathcal{D} (w^{t-1})\|_2^2.
\end{multline*}
Thus we obtain the following by taking $\eta=\frac{1}{\beta}$
\begin{align}
    & \|w^{t-1}-\eta \nabla L_\mathcal{D} (w^{t-1})-w^*\|^2_2 \leq  \nonumber \\
    &(1-\frac{2\alpha}{\alpha+\beta})\|w^{t-1}-w^*\|_2^2-\frac{2}{\beta(\beta+\alpha)}\|\nabla L_\mathcal{D} (w^{t-1})\|_2^2 \nonumber \\
    &+ \frac{1}{\beta^2}\|\nabla L_\mathcal{D} (w^{t-1})\|_2^2 \nonumber  \\
    &\leq (1-\frac{2\alpha}{\alpha+\beta})\|w^{t-1}-w^*\|_2^2.  \label{aeq:8}
\end{align}
Taking the expectation w.r.t $Z_{t-1}$ and using the inequality of $\sqrt{1-x}\leq 1-\frac{x}{2}$ and Lemma 4, we have 
\begin{equation}
    \mathbb{E}\|\hat{w}^t-w^*\|_2\leq (1-\frac{\alpha}{\alpha+\beta})\mathbb{E}\|w^{t-1}-w^*\|_2+O(\frac{\tau d\sqrt{T\log n }}{\beta \sqrt{n}\tilde{\epsilon}}).
\end{equation}
That is,  
\begin{equation*}
    \mathbb{E} \|\hat{w}^T-w^*\|_2\leq (1-\frac{\alpha}{\beta+\alpha})^T\Delta+ O(\frac{\beta}{\alpha}\frac{\tau d\sqrt{T\log n }}{\beta \sqrt{n}\tilde{\epsilon}}). 
\end{equation*}
Thus,  taking $T=O(\frac{\beta}{\alpha}\log n)$, we have the following with probability at least $1-\Omega(\frac{2dn\log n}{(1+n\hat{L}\Delta)^d})$
\begin{equation*}
      \mathbb{E}\|\hat{w}^t-w^*\|_2\leq O(\sqrt{\frac{\beta}{\alpha}}\frac{\Delta \tau d \log n }{\alpha \sqrt{n}\tilde{\epsilon}}). 
\end{equation*}
Since $\tilde{\epsilon}= \sqrt{2\log \frac{1}{\delta}+2\epsilon}-\sqrt{2\log \frac{1}{\delta}} $, by using the Taylor series of the function $\sqrt{x+1}-\sqrt{x}$, we have $\tilde{\epsilon}= O(\frac{\epsilon}{\sqrt{\log \frac{1}{\delta}}})$. 
Since $L_\mathcal{D} (w)$ is $\beta$-smooth we have 
$\mathbb{E}L_\mathcal{D} (w^T)-L_\mathcal{D} (w^*)\leq \frac{\beta}{2}\mathbb{E}\|w^T-w^*\|_2^2$. Thus we  get the proof. 
\end{proof}

\begin{proof}[{\bf Proof of Theorem 7}]
The proof of $(\epsilon, \delta)$-DP is the same as in the proof of Theorem 3. The $\ell_2$ sensitivity is $\frac{s}{n}\frac{4\sqrt{2}}{3}$. 

Next, we show  the upper bound. The key lemma on the uniform converge rate is the following. For convenience, we denote by 
\begin{multline*}
      \hat{g}_j(w)= \frac{1}{n}\sum_{i=1}^n (\nabla_j\ell(w, x_i)\big(1-\frac{\nabla^2_j\ell(w, x_i)}{2s^2\beta}\big)\\
       - \frac{\nabla^3_j\ell(w, x_i)}{6s^2})+\frac{1}{n}\sum_{i=1}^nC\left(\frac{\nabla_j\ell(w, x_i)}{s}, \frac{|\nabla_j\ell(w, x_i)|}{s\sqrt{\beta}}\right) 
\end{multline*}
and $\hat{g}_j(w)=(\hat{g}_1(w), \hat{g}_2(w), \cdots, \hat{g}_d(w))$.
\begin{lemma}[Lemma 8 in \citep{holland2019a}]
Under Assumptions 1 and 4,  with probability at least $1-\delta'$, the following holds for any $w\in \mathcal{W}$, 
\begin{equation}
    \|\hat{g}_j(w)-\mathbb{E}[\nabla \ell(w, x)]\|_2\leq O(\frac{\beta d\sqrt{v\log (\frac{1}{\delta'}\Delta n )}}{\sqrt{n}}). 
\end{equation}
\end{lemma}

Thus, we have the following lemma.
\begin{lemma}\label{alemma:8}
Under the assumptions in the previous lemma, the following holds with probability at least $1-2\delta'$ for any $w\in \mathcal{W}$
\begin{equation}\label{aeq:31}
    \|g_j(w)-\mathbb{E}[\nabla \ell(w, x)]\|_2\leq O(\frac{\beta d\sqrt{vT\log (\frac{1}{\delta'}\Delta n) }}{\sqrt{n}\sqrt{\tilde{\epsilon}}}). 
\end{equation}
\end{lemma}

The remaining proof is almost the same as the proof of Theorem 5 by using Lemma \ref{alemma:8}. We omit it here for convenience.
\end{proof}

\begin{proof}[{\bf Proof of Theorem 8}] 
Let $\hat{w}^t$ denote the same notation as in the proof of Theorem 5. Then, we have
\begin{align*}
    &\|\hat{w}^t-w^*\|_2\leq \|w^{t-1}-\eta g^{t-1}(w^{t-1})-w^*\|_2 \\
    &\leq \|w^{t-1}-\eta \nabla L_\mathcal{D} (w^{t-1})-w^*\|_2\\
    &+\eta \| g^{t-1}(w^{t-1})- L_\mathcal{D} (w^{t-1})\|_2, 
\end{align*}
and 
\begin{align*}
    &\|w^{t-1}-\eta \nabla L_\mathcal{D} (w^{t-1})-w^*\|^2_2\leq \|w^{t-1}-w^*\|_2^2\\
    &-2\eta\langle \nabla L_\mathcal{D}(w^{t-1}), w^{t-1}-w^*\rangle +\eta^2\|\nabla L_\mathcal{D}(w^{t-1})\|_2^2 \\
    &\leq \|w^{t-1}-w^*\|_2^2-2\eta \frac{1}{\beta}\|\nabla L_\mathcal{D}(w^{t-1})\|_2^2+\eta^2 \|\nabla L_\mathcal{D}(w^{t-1})\|_2^2\\
    &\leq \|w^{t-1}-w^*\|_2^2. 
\end{align*}

Thus by Lemma \ref{alemma:8}  we have with probability at least $1-2\delta'$
\begin{equation}\label{eq:32}
   \|\hat{w}^t-w^*\|_2\leq \|w^{t-1}-w^*\|_2+ O(\frac{ d\sqrt{vT\log (\frac{1}{\delta'}\Delta n) }}{\sqrt{n}\sqrt{\tilde{\epsilon}}}). 
\end{equation}
Hence, when $O(\frac{ dT\sqrt{vT\log (\frac{1}{\delta'}\Delta n) }}{\sqrt{n}\sqrt{\tilde{\epsilon}}})\leq \|w^0-w^*\|_2$, we have $\hat{w}^t\in \mathcal{W}$ for all $t=\{1, \cdots, T\}$ with probability at least $1-2\delta' T$.  This means that $\hat{w}^t=w^t$ for all $t\in [T]$. Hence, we proceed to study the algorithm without projection. Let  $D_t= \|w^0-w^*\|_2+ O(\frac{dt\sqrt{vT\log (\frac{1}{\delta'}\Delta n) }}{\sqrt{n}\sqrt{\tilde{\epsilon}}})$ for $t=\{0,1, \cdots, T\}$. By the  smoothness of $L_\mathcal{D}(\cdot)$ we have 
\begin{align*}
   & L_{\mathcal{D}}(w^t)\leq   L_{\mathcal{D}}(w^{t-1})+\langle  \nabla L_{\mathcal{D}}(w^{t-1}), w^t-w^{t-1}\rangle \\ &+\frac{\beta}{2}\|w^t-w^{t-1}\|_2^2 \\
    &= L_{\mathcal{D}}(w^{t-1})+ \eta \langle  \nabla L_{\mathcal{D}}(w^{t-1}), -g^{t-1}(w^{t-1})+ \nabla L_{\mathcal{D}}(w^{t-1}) \\
    &- \nabla L_{\mathcal{D}}(w^{t-1})\rangle +\eta^2 \frac{\beta}{2}\| g^{t-1}(w^{t-1})- \nabla L_{\mathcal{D}}(w^{t-1}) \\
    &+ \nabla L_{\mathcal{D}}(w^{t-1})\|_2^2.
\end{align*}
Since $\eta=\frac{1}{\beta}$, by simple calculation we have 
\begin{multline}\label{aeq:33}
     L_{\mathcal{D}}(w^t)\leq   L_{\mathcal{D}}(w^{t-1})-\frac{1}{2\beta}\|\nabla L_\mathcal (w^{t-1})\|^2 \\ + O(\frac{\beta d^2vT\log (\frac{1}{\delta'}\Delta n) }{n{\tilde{\epsilon}}}). 
\end{multline}
Next we show the following lemma 
\begin{lemma}
Assume that events (\ref{aeq:31}) hold for all $t=\{1, \cdots ,T\}$. Then there exists at least one $t\in \{1, \cdots ,T\}$ such that 
\begin{equation*}
    L_\mathcal{D}(w^t)-L_\mathcal{D} (w^*)\leq 16D_0\chi, 
\end{equation*}
where $\chi=O(\frac{\beta d\sqrt{vT\log(\frac{1}{\delta'}\Delta n)}}{\sqrt{n}\sqrt{\tilde{\epsilon}}})$.
\end{lemma}

\begin{proof}
We note that $D_t\leq 2D_0$ for all $t=0, \cdots , T$. Thus we have 
\begin{multline*}
    L_\mathcal{D} (w)-L_\mathcal{D}(w^*)\leq \|\nabla  L_\mathcal{D} (w)\|_2\|w-w^*\|_2,
\end{multline*}
which implies that 
\begin{equation*}
     \|\nabla  L_\mathcal{D} (w)\|_2\geq \frac{ L_\mathcal{D} (w)-L_\mathcal{D}(w^*)}{\|w-w^*\|_2}.
\end{equation*}
Suppose that there exists $t\in \{1, 2, \cdots ,T\}$ such that $\|\nabla L_\mathcal{D} (w^t)\|_2<\sqrt{2}\chi$. Then, we have $L_\mathcal{D} (w^t)-L_\mathcal{D}(w^*)\leq \|\nabla  L_\mathcal{D} (w^t)\|_2\|w^t-w^*\|_2\leq 2\sqrt{2}D_0\chi$. 

Otherwise suppose that for all $\{1, 2, \cdots , T\}$, $\|\nabla L_\mathcal{D} (w^t)\geq \sqrt{2}\chi$. Then, we have the following for all $t\leq T$, 
\begin{align*}
    &L_\mathcal{D} (w^t)-   L_\mathcal{D} (w^*)\leq L_\mathcal{D} (w^{t-1})- L_\mathcal{D} (w^*) \\
    &-\frac{1}{4\beta}\|\nabla L_\mathcal{D} (w^{t-1})\|_2^2 \\
    &\leq L_\mathcal{D} (w^{t-1})- L_\mathcal{D} (w^*) -\frac{1}{4\beta D_{t-1}^2} ( L_\mathcal{D} (w^{t-1})-L_\mathcal{D}(w^*)).
\end{align*}
Multiplying both side by $[(L_\mathcal{D} (w^t)-   L_\mathcal{D} (w^*))(L_\mathcal{D} (w^{t-1})- L_\mathcal{D} (w^*))]^{-1}$ we get 
\begin{align*}
    & \frac{1}{L_\mathcal{D} (w^t)-   L_\mathcal{D} (w^*)}\geq \frac{1}{L_\mathcal{D} (w^{t-1})- L_\mathcal{D} (w^*)} \\
    &+\frac{1}{4\beta D_{t-1}^2}\frac{L_\mathcal{D} (w^{t-1})- L_\mathcal{D} (w^*)}{L_\mathcal{D} (w^t)-   L_\mathcal{D} (w^*)}\\
    &\geq \frac{1}{L_\mathcal{D} (w^{t-1})- L_\mathcal{D} (w^*)} +\frac{1}{16\beta D_0^2},
\end{align*}
where the last inequality is due to the facts that $D_t\leq 2D_0$ and $L_\mathcal{D}(w^{t-1})\geq L_\mathcal{D}(w^{t})$.

Hence, we have 
\begin{equation}\label{aeq:34}
     \frac{1}{L_\mathcal{D} (w^T)-   L_\mathcal{D} (w^*)}\geq \frac{T}{16\beta D_0^2}\geq \frac{1}{16 D_0\chi}
\end{equation}
using the fact that $T=\frac{\beta D_0}{\chi}$, that is, $T=\tilde{O}\left(\frac{\|w^0-w^*\|_2\sqrt{n}\sqrt{\tilde{\epsilon}}}{d}\right)^\frac{2}{3}$. Thus $\chi = \tilde{O}(\Delta\frac{d^\frac{2}{3}}{(n\tilde{\epsilon})^\frac{1}{3}})$. 
\end{proof}

Next we  show that 
\begin{equation}
    L_\mathcal{D} (w^T)-L_\mathcal{D} (w^*)\leq 16D_0\chi+\frac{1}{2\beta}\chi^2. 
\end{equation}
Let $t=t_0$ be the first time that $ L_\mathcal{D} (w^T)-L_\mathcal{D} (w^*)\leq 16D_0\chi $. We show that for any $t\geq t_0$, $ L_\mathcal{D} (w^t)-L_\mathcal{D} (w^*)\leq 16D_0\chi+\frac{1}{2\beta}\chi^2$. If not,  let $t_1$ be the first time that $ L_\mathcal{D} (w^t)-L_\mathcal{D} (w^*)>16D_0\chi+\frac{1}{2\beta}\chi^2$. Then, we must have $L_\mathcal{D} (w^{t_1})> L_\mathcal{D} (w^{t_1-1})$. By (\ref{aeq:33}) we have 
\begin{multline*}
    L_\mathcal{D} (w^{t_1-1})-L_\mathcal{D} (w^*)\geq\\ L_\mathcal{D} (w^{t_1})-L_\mathcal{D} (w^*)-\frac{1}{2\beta}\chi^2>16D_0\chi.
\end{multline*}
Thus, we have 
\begin{equation*}
    \|\nabla    L_\mathcal{D} (w^{t_1-1})\|_2\geq \frac{ L_\mathcal{D} (w^{t_1-1})-L_\mathcal{D} (w^*)}{\|w^{t_1-1}-w^*\|_2}\geq 8\chi.
\end{equation*}
By (\ref{aeq:33}) we have $L_\mathcal{D} (w^{t_1})\leq L_\mathcal{D} (w^{t_1-1})$ which is a contradiction. 
\end{proof}

\section{Explicit Form of $C(a,b)$ in (10)}
We first define the following notations:
\begin{align}
   &V_- := \frac{\sqrt{2}-a}{b}, V_{+}=\frac{\sqrt{2}+a}{b} \\
   &F_{-}:= \Phi(-V_-), F_{+}:=\Phi(-V_+) \\
   &E_{-}:= \exp(-\frac{V^2_-}{2}), E_{+}:=\exp(-\frac{V^2_{+}}{2}),
\end{align}
where $\Phi$ denotes the CDF of the standard Gaussian distribution. Then \begin{equation}
    C(a,b)=T_1+T_2+\cdots+T_5, 
\end{equation}
where 
\begin{align}
    &T_1:= \frac{2\sqrt{2}}{3}(F_{-}-F_{+}) \\
    & T_2:= -(a-\frac{a^3}{6})(F_{-}+F_{+}) \\ 
    & T_3:=\frac{b}{\sqrt{2\pi}}(1-\frac{a^2}{2})(E_{+}-E_{-})\\ 
    &T_4 : = \frac{ab^2}{2}\left(F_{+}+F_{-}+\frac{1}{\sqrt{2\pi}}(V_{+}E_{+}+V_{-}E_{-})\right) \\
    & T_5:= \frac{b^3}{6\sqrt{2\pi}}\left((2+V_{-}^2)E_{-}-(2+V_{+}^2)E_{+}\right). 
\end{align}

\section{Full description of experiments}
For the synthetic data generation, we select the parameters $(\mu=1,\sigma=1)$ and $(\mu=0.2,\sigma=0.2)$ for the Lognormal and  Loglogistic noises underlying, respectively. The step size of Algorithm 3 is set to 0.01 where $m=0.05n$.  As for algorithm 4, $v=5$, failure probability $\delta'=0.01$ and the step size is set to $0.1$. For the stochastic Algorithm 4, the step size is selected as $\frac{1}{\sqrt{t}}$, where $t$ is the iteration number. Accordingly, $\bar{w}^T = \frac{\sum_{t=1}^T w^t}{T}.$ Corresponding to Fig. 1 and 2, we present the results which also mark the difference between the best and the worst performances as follows.
\begin{figure*}[!htbp]
    \centering
    \begin{subfigure}[b]{.24\textwidth}
    \includegraphics[width=\textwidth,height=0.15\textheight]{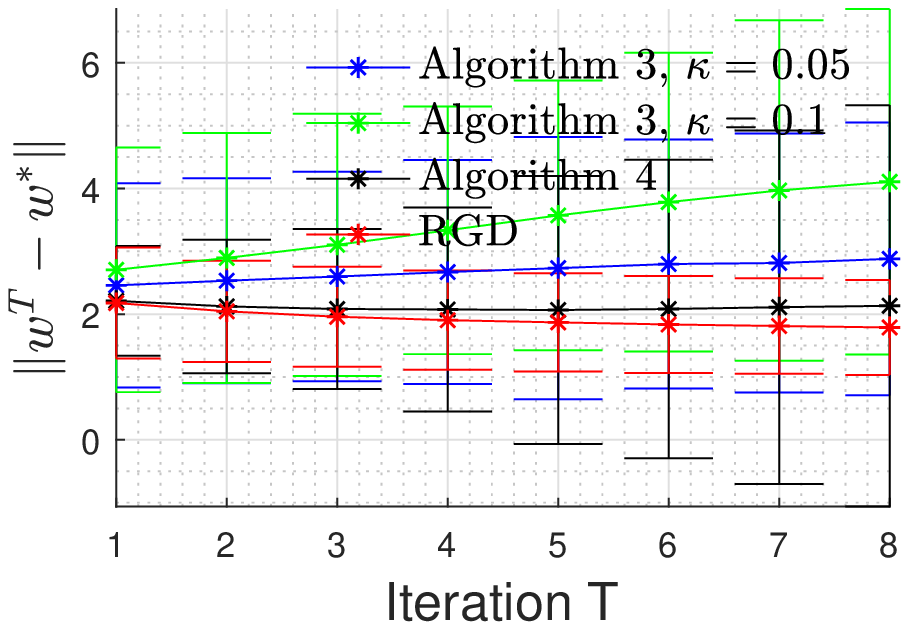}
    \caption{$\epsilon=1$ \label{fig1:a}}
    \end{subfigure}
    ~
    \begin{subfigure}[b]{.24\textwidth}
    \includegraphics[width=\textwidth,height=0.15\textheight]{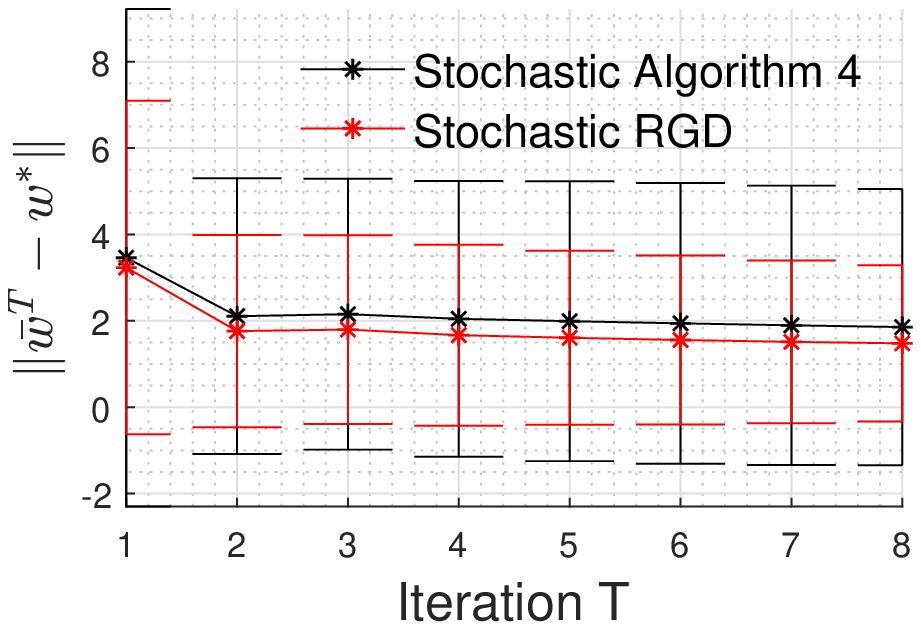}
    \caption{$\epsilon=0.5$ \label{fig1:b}}
    \end{subfigure}    
     \begin{subfigure}[b]{.24\textwidth}
    \includegraphics[width=\textwidth,height=0.15\textheight]{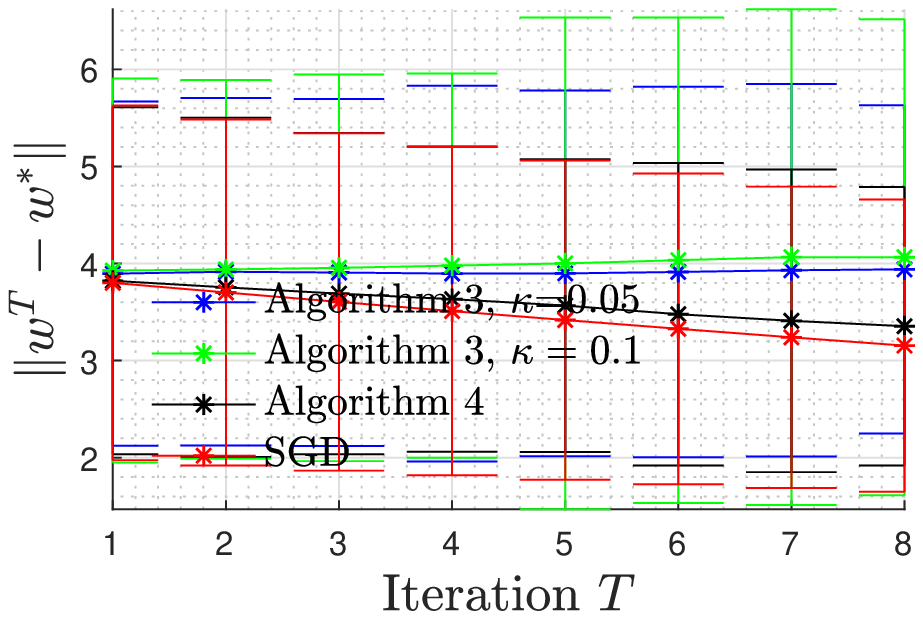}
    \caption{$\epsilon=1$ \label{fig1:c}}
    \end{subfigure}
    ~
    \begin{subfigure}[b]{.24\textwidth}
    \includegraphics[width=\textwidth,height=0.15\textheight]{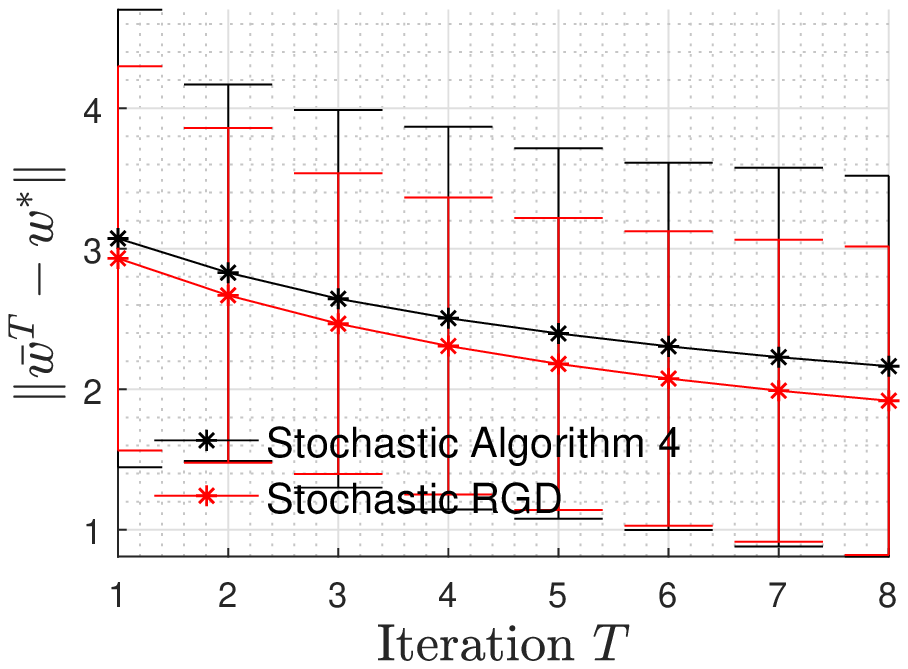}
    \caption{$\epsilon=0.5$ \label{fig1:d}}
    \end{subfigure}  
    \caption{Experiments on synthetic datasets. Figures (a) and (b) are for ridge regressions over synthetic data with Lognormal noises. Figures (c) and (d) are for logistic regressions over synthetic data with Loglogistic noises. \label{fig:1bar} }
\end{figure*}
\begin{figure*}[!htbp]
    \centering
    \begin{subfigure}[b]{.24\textwidth}
    \includegraphics[width=\textwidth,height=0.15\textheight]{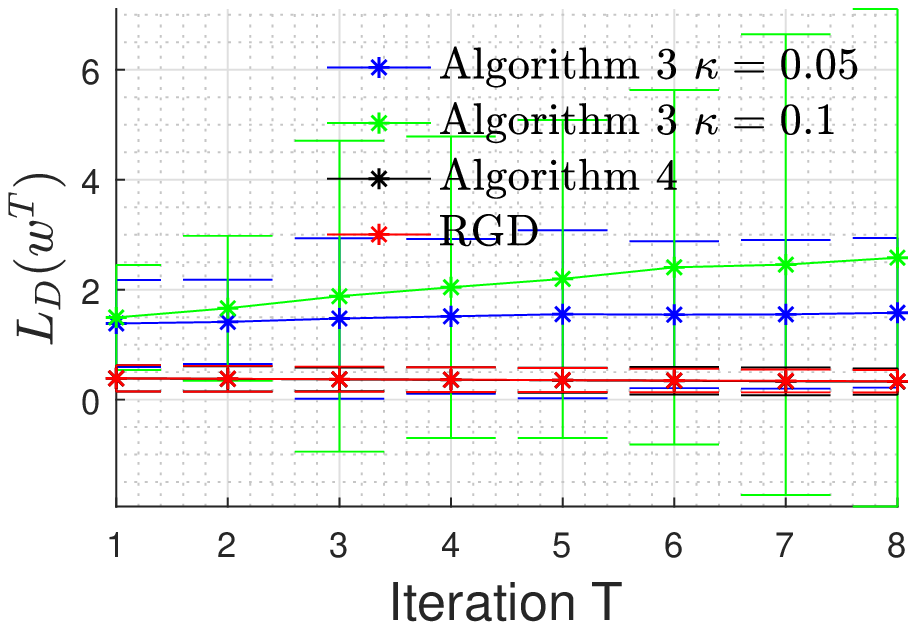}
    \caption{$\epsilon=1$ \label{fig2:a}}
    \end{subfigure}
    ~
    \begin{subfigure}[b]{.24\textwidth}
    \includegraphics[width=\textwidth,height=0.15\textheight]{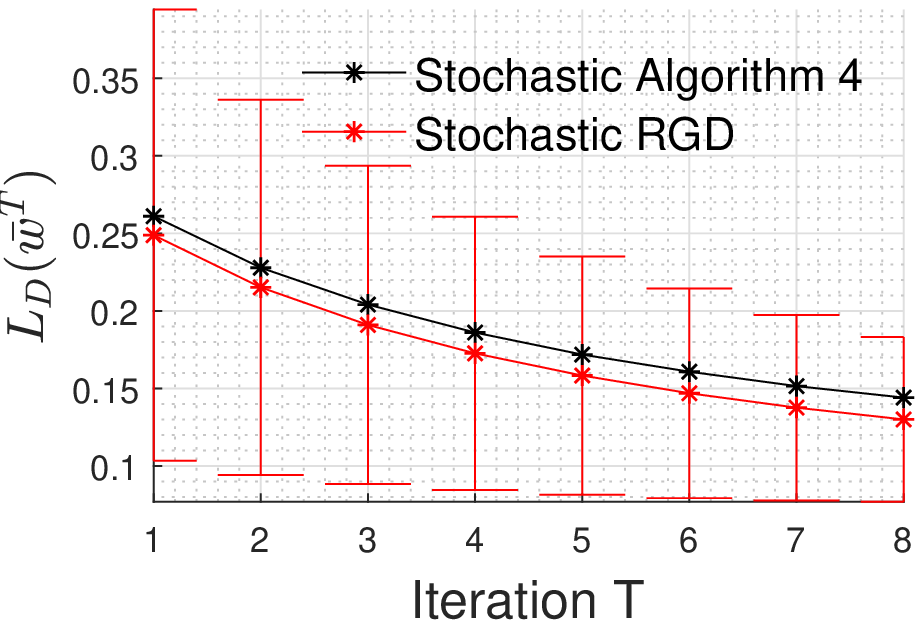}
    \caption{$\epsilon=0.5$ \label{fig2:b}}
    \end{subfigure}    
       \begin{subfigure}[b]{.24\textwidth}
    \includegraphics[width=\textwidth,height=0.15\textheight]{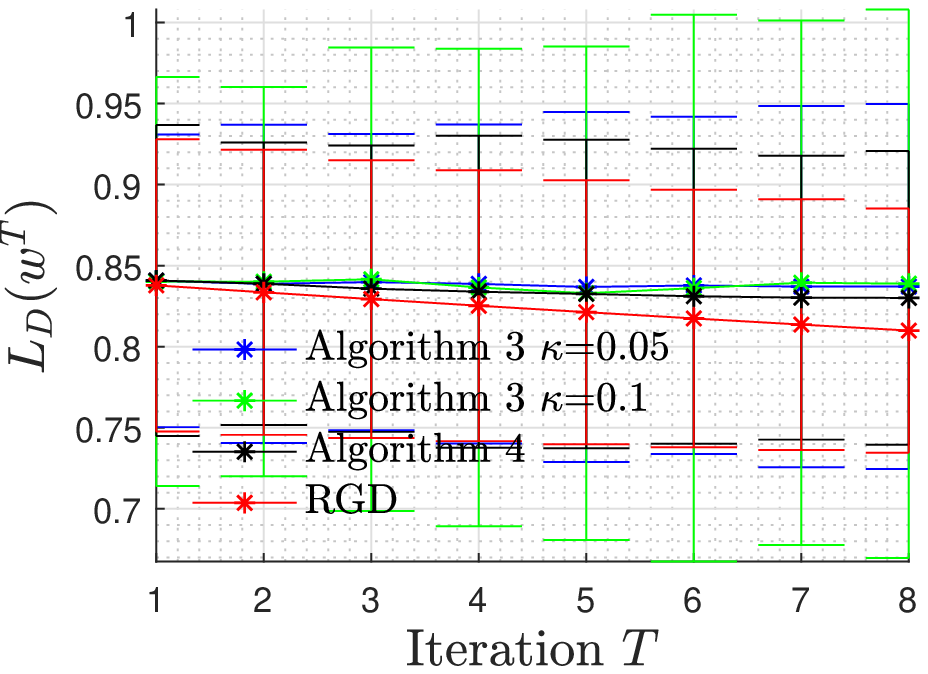}
    \caption{$\epsilon=1$ \label{fig2:c}}
    \end{subfigure}
    ~
    \begin{subfigure}[b]{.24\textwidth}
    \includegraphics[width=\textwidth,height=0.15\textheight]{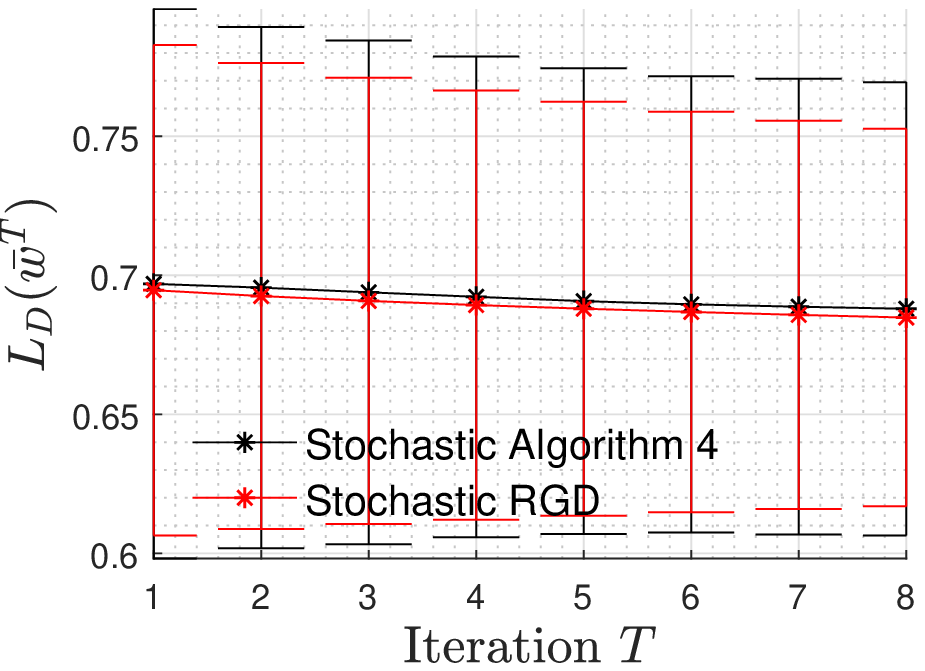}
    \caption{$\epsilon=0.5$ \label{fig2:d}}
    \end{subfigure}    
    
    \caption*{Experiments on UCI Adult dataset. Figures (a) and (b) are for ridge regressions. Figures (c) and (d) are for logistic regressions. \label{fig:2} }
\end{figure*}

To measure the impact from dimension on performances, we fix $n=10^5$ and test $d$ varying from $10$ to $50$ through stochastic Algorithm 4 and RGD under the same setup as above. To test the impact from the size of the dataset, we fix $d=20$ and test $n$ varying from $2\times10^4$ to $10^5$. 


\end{document}